\documentclass[letterpaper]{article} % DO NOT CHANGE THIS
\usepackage{aaai2026}  % DO NOT CHANGE THIS
\usepackage{times}  % DO NOT CHANGE THIS
\usepackage{helvet}  % DO NOT CHANGE THIS
\usepackage{courier}  % DO NOT CHANGE THIS
\usepackage[hyphens]{url}  % DO NOT CHANGE THIS
\usepackage{graphicx} % DO NOT CHANGE THIS
\usepackage{natbib}  % DO NOT CHANGE THIS AND DO NOT ADD ANY OPTIONS TO IT
\usepackage{caption} % DO NOT CHANGE THIS AND DO NOT ADD ANY OPTIONS TO IT
\usepackage{algorithm}
\usepackage{algorithmic}
\usepackage{amsfonts}
\usepackage{amsmath} % Added for \eqref command
\usepackage{booktabs}
\usepackage{array}
\usepackage{subcaption}
\usepackage{float}
\usepackage{fancyhdr}

\newtheorem{lemma}{Lemma}
\newenvironment{proof}[1][Proof]{\par\vspace{1ex}\noindent\textbf{#1. }\ignorespaces}{\hfill$\square$}
\usepackage{newfloat}
\usepackage{listings}
\DeclareCaptionStyle{ruled}{labelfont=normalfont,labelsep=colon,strut=off} % DO NOT CHANGE THIS
\floatstyle{ruled}
\newfloat{listing}{tb}{lst}{}
\floatname{listing}{Listing}
\title{Bi‐Level Contextual Bandits for Individualized Resource Allocation under Delayed Feedback$^{\dagger}$}
\author{
    Mohammadsina Almasi, Hadis Anahideh\thanks{Corresponding Author.}\\
}
\affiliations{
    University of Illinois Chicago\\
    842 W Taylor St.\\
    Chicago, IL 60607 USA\\
    \texttt{\{malmas6,hadis\}@uic.edu} \\
}

\usepackage{bibentry}
\nocopyright  % <-- add this line for the arXiv/preprint version ONLY

\begin{document}
% \setlength{\headheight}{16pt}   % >= 14pt avoids fancyhdr warnings
% \setlength{\headsep}{14pt}
% \fancyhf{} % clear
% \chead{\small The 40th AAAI Conference on Artificial Intelligence (AAAI–26)}
% \renewcommand{\headrulewidth}{0.4pt}
% \pagestyle{fancy}         % all pages
% % If you also want it on the title page:
% \thispagestyle{fancy}

\maketitle

\begingroup
\renewcommand\thefootnote{\fnsymbol{footnote}}
\footnotetext[2]{Preprint. Accepted at AAAI-26 (AISI Track). The final, typeset
version will appear in the Proceedings of the AAAI Conference on Artificial Intelligence
(AAAI-26), 2026.}
\endgroup

\begin{abstract}
Equitably allocating limited resources in high-stakes domains—such as education, employment, and healthcare—requires balancing short-term utility with long-term impact, while accounting for delayed outcomes, hidden heterogeneity, and ethical constraints. However, most learning-based allocation frameworks either assume immediate feedback or ignore the complex interplay between individual characteristics and intervention dynamics. We propose a novel bi-level contextual bandit framework for individualized resource allocation under delayed feedback, designed to operate in real-world settings with dynamic populations, capacity constraints, and time-sensitive impact. At the meta level, the model optimizes subgroup-level budget allocations to satisfy fairness and operational constraints. At the base level, it identifies the most responsive individuals within each group using a neural network trained on observational data, while respecting cooldown windows and delayed treatment effects modeled via resource-specific delay kernels. By explicitly modeling temporal dynamics and feedback delays, the algorithm continually refines its policy as new data arrive, enabling more responsive and adaptive decision-making. We validate our approach on two real-world datasets from education and workforce development, showing that it achieves higher cumulative outcomes, better adapts to delay structures, and ensures equitable distribution across subgroups. Our results highlight the potential of delay-aware, data-driven decision-making systems to improve institutional policy and social welfare.
\end{abstract}

\begin{links}
    \link{Code}{https://github.com/sinatorrr/MAB}
    % \link{Datasets}{https://github.com/sinatorrr/MAB}
    % \link{Extended version}{https://aaai.org/example/extended-version}
\end{links}

\section{Introduction} \label{sec:introduction}

Resource allocation is a central challenge in high-stake domains such as healthcare, where practitioners determine treatment priorities~\cite{lane2017equity,aktacs2007decision,daniels2016resource}; telecommunications, where capacity must be distributed across competing channels~\cite{su2019resource,hui2002resource,gibney1998dynamic}; education, where instructional or financial resources are allocated to students~\cite{monk1981toward,liefner2003funding,massy1996resource}; and social welfare, where job training and support programs are delivered~\cite{nguyen2014computational,roos2010complexity}. In these settings, decision-makers often observe contextual information at the individual level and must sequentially allocate scarce interventions to optimize long-term, population-wide outcomes~\cite{hegazy1999optimization,gong2012efficient}.

Classical resource allocation models provide useful abstractions but rely on idealized assumptions, often overlooking temporal, ethical, and institutional constraints that affect fairness, feasibility, and policy relevance in real-world settings~\cite{wang2022surrogate,zou2019reinforcement,obermeyer2019dissecting,chouldechova2017fair}. To address these limitations, recent research has leveraged algorithmic frameworks such as multi-armed bandits (MABs), particularly their contextual variants, to make personalized, adaptive decisions based on observed features~\cite{grover2018best, gyorgy2021adapting,joulani2013online}. However, several critical challenges remain unaddressed in the literature.

First, most existing MAB approaches assume that outcomes are observed immediately following an allocation. In reality, the effects of interventions unfold gradually: 
medical treatments manifest their efficacy over days or weeks~\cite{hanna2020mortality,yanovski2014long}, educational interventions accrue impact over semesters~\cite{barnett1995long,almalki2022effect}, and workforce programs influence long-term employment trajectories~\cite{edmondson2019co}. These delayed and temporally structured effects introduce feedback dynamics that are rarely modeled in full. While recent methods incorporate delay via fixed or stochastic lags~\cite{lancewicki2021stochastic,shi2023statistical}, most treat delay as a nuisance rather than learning the temporal impact profile of each intervention.
To overcome these limitations, studies have introduced delay-aware allocation methods, including post hoc reward adjustments to capture deferred effects~\cite{tang2021bandit}. Another line of work uses an episodic framework, modeling feedback delays as discrete random variables representing decision rounds between action and outcome. However, these models often assume fixed or context-independent lags, limiting their ability to capture heterogeneous, intervention-specific temporal dynamics~\cite{kuang2023posterior,yin2023long}.

Second, traditional models assume a static population and ignore real-world deployment constraints. In practice, participants join and leave in cohort cycles (e.g., semesters or enrollment periods), creating time-varying populations that challenge fixed decision-pool assumptions. They also overlook ethical and institutional rules like cooldown periods restricting repeated allocations of the same resource to the same individual—constraints essential for fairness and feasibility~\cite{li2022efficient,patil2021achieving}. While some studies have begun to incorporate these elements, they typically do so under highly stylized conditions—assuming zero delay, a single homogeneous resource type, or i.i.d. reward structures~\cite{wang2019distributed,zuo2021combinatorial,burnetas2025optimal}. Such simplifications fail to capture the structural dependencies and heterogeneity that characterize real-world allocation environments.
Most existing algorithms target either individual personalization or group fairness, rarely both. Some optimize personalized rewards but ignore group equity; others enforce fairness while overlooking individual heterogeneity~\cite{li2022efficient,patil2021achieving}. Bridging both is essential for equitable, effective policy in settings requiring individual adaptation and group awareness.

In light of these challenges, we propose \textbf{Meta-level Contextual Upper Confidence Bandit (MetaCUB)} a novel {bi-level contextual bandit framework} for {individualized resource allocation} under delayed feedback and real-world constraints. At its core is a neural network that maps individual-level features to subgroup-level treatment effects, capturing latent heterogeneity in responsiveness. The meta level allocates sub-budgets across groups to ensure equity, while the base level selects the most responsive individuals within each group under resource-specific constraints such as cooldowns and budgets. Our framework models cohort-driven population dynamics, multiple resource types with distinct budgets, and heterogeneous delay-feedback profiles, where each resource has a delay kernel describing its temporal impact. Evaluations on real-world datasets show that \textbf{MetaCUB} learns effective, constraint-aware allocation policies that outperform strong baselines.

In summary, our contributions are fourfold. First, we propose a neural network–based learning framework that maps individual contexts to subgroup effects, capturing latent heterogeneity and improving regret over linear models. Second, we design a bi-level contextual bandit that jointly optimizes group- and individual-level allocations under fairness, cooldown, and capacity constraints. Third, we build a deployment-ready architecture modeling cohort dynamics, resource-specific budgets, delay kernels, and stochastic cooldowns. Finally, we validate our framework through extensive experiments on real-world datasets, showing superior cumulative reward, fairness, and delay adaptation over state-of-the-art baselines.

\section{Related Works} \label{sec:related_works}

The multi-armed bandit (MAB) framework is a foundational model for sequential decision-making under uncertainty, widely used for efficient resource allocation in dynamic settings~\cite{kuleshov2014algorithms,agrawal2012analysis}. It has shown effectiveness in domains like telecommunications, finance, online platforms, and healthcare, where adaptive learning is critical~\cite{huo2017risk,biswas2021learn,bouneffouf2020survey}.

Contextual bandits extend classical MABs by leveraging individual-level covariates to enable personalized decision policies~\cite{lu2010contextual,zhou2015survey}. Methods like LinUCB and contextual Thompson Sampling provide theoretical regret guarantees when rewards depend linearly or probabilistically on covariates~\cite{agrawal2013thompson,kaufmann2012bayesian,chouldechova2017fair}. These models have advanced personalization in areas such as online recommendations, clinical decisions, and educational interventions. Recent extensions to multi-agent and group-based learning enable collaborative and distributed allocation across multiple learners or subpopulations~\cite{cui2019multi,xu2020collaborative}.

Despite this progress, most contextual bandit approaches remain difficult to deploy effectively in socially impactful domains. First, they often assume immediate feedback, overlooking that real-world interventions—such as tutoring, job training, or healthcare treatments—produce delayed effects over time. Although recent studies address stochastic or bounded delays~\cite{joulani2013online,gael2020stochastic}, they often treat delay as a nuisance variable, neglecting its temporal dynamics. Several works have proposed tracking reward queues~\cite{tang2021bandit,vernade2017stochastic}, bounding adversarial regret~\cite{erez2024regret,steiger2022learning}, coupling delay with payoff magnitude~\cite{schlisselberg2025delay} but few frameworks model how reward signals are distributed across time in a resource-specific and learnable manner.
Second, while constrained MABs, such as bandits with knapsacks~\cite{badanidiyuru2018bandits,tran2012knapsack} or fairness-aware variants~\cite{chen2020fair,claure2020multi}, address limitations on budgets or equity, they often assume static populations and homogeneous reward structures. In practice, many allocation settings involve dynamic cohorts, where individuals enter and exit over time (e.g., educational semesters or batched workforce programs). Some studies have examined the concept of a dynamic population, which captures the changing availability of arms, through frameworks such as contextual combinatorial bandits with volatile arms and submodular rewards, or interest-drift models with immediate feedback; however, these approaches overlook the partial observability of feedback inherent in educational and workforce outcomes~\cite{chen2018contextual,xu2020contextual}. Additionally, real-world deployments must satisfy cooldown constraints, such as ethical limits on repeated treatment~\cite{liu2012learning,chen2022uncertainty,mate2022field}, yet such pacing mechanisms are seldom integrated into existing MAB formulations; for example, classical blocking methods address availability pacing but ignore delayed impact~\cite{basu2021contextual}.

Third, prior works often separate fairness from personalization. Some optimize individual outcomes without ensuring group equity, while others enforce group fairness with no within-group heterogeneity. Few models integrate both, enabling individualized treatment within group-level budget constraints under delayed feedback and dynamic population.

\section{Problem Setup}\label{sec:methodology}

We consider a sequential decision-making problem where a central planner (e.g., policymaker or service provider) allocates limited resources over time to individuals grouped by demographic or socioeconomic traits. The goal is to optimize long-term outcomes—like academic or employment success—via adaptive, context-aware decisions accounting for delayed effects and real-world constraints. This setting appears in high-stakes domains such as education (e.g., financial aid, tutoring) and workforce programs (e.g., training support). Our goal is to design an allocation framework for evolving, constrained, and fairness-sensitive environments.

Consider a set of $N$ individuals of $K$ demographic subgroups of size $n^k$ each and $N=\sum^{K}_{k=1} n^k$. Each individual is associated with a context vector $\mathbf{x}^i\in \mathcal{X}\subseteq\mathbb{R}^{M}$, capturing $M$ demographic and domain-specific attributes, including their recent resource assignments. The decision-maker manages $R$ distinct resource types, each with an integer-valued budget $b^r \in \mathbb{Z}_+$ for $r \in R$. The allocation process unfolds over $T$ discrete time steps. Let $\mathcal{I} \subseteq N$ denote the subset of individuals who receive at least one allocation during the decision horizon. At each decision round $t\in T$, one individual $i_t \in \mathcal{I}$ is selected and assigned a unit of resource $r_t\in R$, subject to the budget constraint:
\begin{equation}\label{eq:budget_const}
    \sum_{t=1}^{T}\mathbb{I}\{r_t=r\}\;\le\;b^{r} \quad\forall\,r\in R
\end{equation}
At each round $t$, the action taken is the pair $a_t = (i_t, r_t)$, where $i_t \in \mathcal{I}$ is the selected individual and $r_t \in R$ is the assigned resource. The full allocation sequence over the horizon is denoted by $\{a_t\}_{t=1}^{T}$. Allocating a resource to an individual yields an instantaneous reward $y(t) = f\bigl(\mathbf{x}^i(t)\bigr)$, where $f : \mathcal{X} \to \mathbb{R}$ is a reward function mapping the individual's context, possibly including the allocated resource, to an expected outcome\footnote{The reward function can be adapted to the application domain; for binary outcomes, for example, one may use $f : \mathcal{X} \to [0, 1]$.}. The decision-maker's objective is to select a sequence of actions $\{a_t = (i_t, r_t)\}_{t=1}^{T}$ that maximizes the expected cumulative reward over the time horizon:
\begin{equation}\label{eq:base_obj}
\max_{\{a_t\}_{t=1}^T}\
\mathbb{E}\biggl[\sum_{t=1}^T y(t)\biggr]
\end{equation}
Equivalently, the goal can be framed as minimizing the cumulative regret relative to the best feasible allocation policy in hindsight, i.e., the optimal policy that would have been chosen with full knowledge of individual responses.

A natural and flexible framework for modeling sequential resource allocation under uncertainty is the contextual multi-armed bandit (MAB)~\cite{lu2010contextual}
In this setting, each feasible action—defined as an individual–resource pair $a_t=(i_t,r_t)$ is treated as an arm. The associated context $\mathbf{x}^{i}(t)$ provides side information about the individual and their history, while a pretrained reward function $f(\cdot)$ serves as the feedback model, predicting expected outcomes for each action. Resource budgets impose constraints on the number of allowable arm pulls, and the objective becomes minimizing regret relative to the best allocation policy in hindsight
A common way to balance this trade-off is via Upper Confidence Bound (UCB) methods, which choose actions maximizing predicted reward plus an uncertainty bonus, balancing exploration and exploitation. A baseline constrained contextual MAB procedure is outlined in Algorithm 3 in Appendix. We extend it with additional components to better capture real-world complexities.

\textbf{Population Change.} In many deployment settings, participants enroll and exit in fixed-duration cycles, forming successive distinct cohorts. To capture this, we partition the decision horizon of $T$ rounds into $H = \lceil T / L\rceil$ contiguous blocks of length $L$. Let $\mathcal{I}_{h}\subseteq N$ denote the set of individuals in cohort $h$, who are eligible to receive allocations only during rounds $t \in [(h-1)L + 1,\;hL], \; h=1,\dots,H$. At the beginning of each block $h$, the decision-maker observes the context vectors $\{\mathbf{x}^{i}(t) : i \in \mathcal{I}_{h}\}$ and allocates resources exclusively among individuals in cohort $\mathcal{I}_{h}$ for the next $L$ rounds. At the end of this period, cohort $\mathcal{I}_{h}$ exits the program and is replaced by the incoming cohort $\mathcal{I}_{h+1}$. 
This structure introduces non-stationarity into the decision process, as the available pool of individuals varies across time. The algorithm must therefore learn not only whom to allocate resources to, but also adapt its policy to the evolving population across cohorts.

\textbf{Delayed Feedback.} Resource allocations in real‑world scenarios often exhibit delayed effects. We model resource‐specific feedback delays over a horizon of $T$ rounds. For each resource $r \in R$, we define a delay kernel $K^r$, a nonnegative function over the time horizon that distributes the realized reward across future rounds. Formally, let $K^{r}:\{0,\dots,T-1\}\rightarrow[0,1],\quad 
\sum_{\tau=0}^{T-1} K^{r}(\tau)=1,$ where $K^r(\tau)$ denotes the proportion of the reward from allocating resource $r$ that is observed $\tau$ rounds after the allocation. By definition, $K^r(\tau) = 0$ for $\tau < 0$ or $\tau > T{-}1$, ensuring bounded support. 
To construct these kernels, we discretize a Beta distribution over $[0,1]$ into $T$ equal-width bins. In particular, 
\begin{equation}\label{eq:delay}
     K^{r}(\tau)
=\int_{\frac{\tau}{T}}^{\frac{\tau+1}{T}}
\mathrm{Beta}(z;\alpha^{r},\beta^{r})\,dz
\end{equation} 
for $\tau=0,1,\dots,T-1$, where $\mathrm{Beta}(z;\alpha,\beta)=\frac{z^{\alpha-1}(1-z)^{\beta-1}}{B(\alpha,\beta)}$ for $z\in(0,1)$, and $B(\alpha,\beta)=\int_{0}^{1}z^{\alpha-1}(1-z)^{\beta-1}\,dz$. This formulation flexibly models feedback latency: $\alpha^r < 1$ gives immediate feedback, $\beta^r < 1$ produces long-tail delays, and $\alpha^r, \beta^r > 1$ yield unimodal kernels. Mixtures of Beta densities can represent more complex or multimodal delays. Discretized Beta kernels assign unit mass to $\{0,\cdots,T-1\}$ and flexibly capture early, late, or long-tailed delays. They attribute outcomes to service rounds without leakage beyond the operational window, aligning with program accounting. For each resource, $(\alpha,\beta)$ parameters are chosen from plausible timing profiles and fixed during learning. The framework remains distribution-agnostic, any normalized delay kernel is admissible, and supports adaptive or meta-learned kernel estimation when greater flexibility is needed~\cite{kassraie2022meta}. Overall, this kernel-based framework extends beyond fixed delays or exponential decay, capturing heterogeneous, resource-dependent feedback dynamics that mirror real-world interventions.

\textbf{Allocation Cooldown.} In real-world interventions, individuals are rarely allowed, or advised, to receive the same resource repeatedly in short intervals~\cite{weiner2012search,legare2018interventions}. Treatment effects take time to manifest, capacity is limited, and regulations often restrict repeated support. To model this, we introduce \emph{cooldown} constraints that prevent reallocation of the same resource to an individual for several rounds after use.

Let $c^r \in \mathbb{Z}_{+}$ with $c^r < T$ denote the cooldown length, i.e., the number of rounds during which an individual is ineligible to receive the same resource $r \in R$ again. When $c^r = T$, the cooldown spans the full horizon, limiting each individual to at most one allocation of $r$ over the $T$ rounds. After individual $i\in \mathcal{I}_h$ receives resource $r\in R$ at round $t\in T$, they become temporarily ineligible to receive the same resource again for the next $c^r$ consecutive rounds. Formally, let $z_{i,r}(t)\in\{0,1\}$ whether resource $r$ is allocated to individual $i$ at round $t$. That is we impose the following constraint:
\begin{equation}\label{eq:cooldown}
\sum_{s=t}^{t+c^r}z_{i,r}(s)\le1 \;
 \forall i\in\mathcal{I}_{h} ,\, \forall r\in R,\,t =1,\dots,T-c^r
\end{equation}
This condition limits each individual to one unit of resource $r$ within any $c^r + 1$ consecutive rounds. While the delay kernel $K^r$ models reward evolution, the cooldown enforces allocation spacing, jointly forming a temporally aware framework that balances impact and pacing.

\section{Proposed Approach}\label{sec:proposed approach}
We extend the base model in Equation~\eqref{eq:base_obj} to include \emph{population change}, \emph{delayed feedback}, and \emph{allocation cooldowns}. Keeping binary decisions $z_{i,r}(t)\in{0,1}$ and budget limits (Equation~\eqref{eq:budget_const}), we add time-varying eligibility for cohort dynamics, cooldowns (Equation~\eqref{eq:cooldown}) to control repeated allocations, and cumulative rewards reflecting temporally distributed effects through resource-specific delay kernels.
At each decision round $t$, the observed reward $y(t)$ aggregates the delayed impacts of all past allocations whose effects materialize at time $t$. Formally, the reward is computed as:
\begin{equation}
    y(t)=\sum_{u=1}^{t}\sum_{i\in \mathcal{I}_h(u)}\sum_{r\in R}
    K^{r}(t-u)\, f\bigl(\mathbf{x}^{\,i}(u)\bigr)\, z_{i,r}(u).
    \label{eq:yt}
\end{equation}
where $K^r(\cdot)$ is the delay kernel associated with resource $r$, and $f(\cdot)$ is a learned model that maps context vectors to predicted outcomes.
The full problem is then formulated as the following constrained optimization program. We use the following shorthand: "1/round" for one resource per individual per round, "B" for total resource budget constraints, and "CD" for cooldown restrictions on repeated allocations.
\setlength{\fboxsep}{3pt}
\setlength{\jot}{0.5pt}
{\small
\noindent\fbox{%
  \parbox{\dimexpr\linewidth-2\fboxsep-2\fboxrule}{%
    \begin{subequations}\label{eq:extended_problem}
    \begin{align}
      &\max_{z_{i,r}(t)}\quad
         \mathbb{E}\left[\sum_{t=1}^{T} y(t)\right]
          \label{obj:extended}\\
      &\text{s.t.}\nonumber\\
      &\sum_{r\in R} z_{i,r}(t)\le 1 \quad  \forall i,t \quad \text{(1/round)}
          \label{cons:oneperround}\\
      & \sum_{t=1}^{T}\sum_{i\in \mathcal{I}_h (t)} z_{i,r}(t)\le b^{r},   \forall r \quad \text{(B)} 
          \label{cons:budget}\\
      & \sum_{s=t}^{t+c^r} z_{i,r}(s)\le 1 \quad \forall i,\, r,\, t=1,\dots,T-c^r \quad \text{(CD)}
          \label{cons:cooldown}\\
      &z_{i,r}(t)\in\{0,1\} \quad 
      \forall i, r, t \nonumber
    \end{align}
    \end{subequations}
  }
}
}

To solve the extended problem formulation in~\eqref{eq:extended_problem}a-d, we propose \textbf{MetaCUB} (Meta level Contextual Upper Confidence Bandit), a bi-level contextual bandit optimization framework. At the upper level, a meta-bandit allocates fractional resource budgets across demographic groups to maximize population-wide impact under equity constraints. At the lower level, an individual bandit selects individuals in each group using contextual features and a learned mapping from profiles to expected outcomes (e.g., mean GPA).

\textbf{Meta-level Framework.}
We assume a fixed computational budget of $T_m$ iterations at the meta-level. At each iteration $t_m \in \{1, \dots, T_m\}$, the meta-level algorithm selects a candidate meta-allocation policy 
$\bar{\boldsymbol z}(t_m) = \left\{ \bar{z}^k_r(t_m) \right\}_{k \in K, r \in R} \in \widetilde{\Delta}^{|K| \cdot |R|}$, where $\Delta^{|K| \cdot |R|}$ denotes the $(|K| \cdot |R|)$-dimensional probability simplex (i.e., )the non-negative vectors in $\mathbb{R}^{|K| \cdot |R|}$ that sum to 1.) ensuring that the total resource allocation across all subgroup–resource pairs remains normalized.
Each entry $\bar{\boldsymbol z}^{k}_{r}(t_m)$ specifies the proportion of the total resource budget allocated to subgroup $k \in K$ for resource type $r \in R$ at iteration $t_m$. The policy must satisfy the following simplex constraint:
\begin{equation}
\sum_{k \in K} \sum_{r \in R} \bar{\boldsymbol z}^{k}_{r}(t_m) \leq 1 \quad \forall t_m\in \{1,...,T_m\} 
\end{equation}
To initiate the optimization, we generate an initial set of $n_0$ candidate meta-allocation policies $\{\bar{\boldsymbol z}^{(j)}\}_{j=1}^{n_0}$ sampled from the interior of the simplex. For each candidate policy $\bar{\boldsymbol z}^{(j)}$, we simulate the subgroup-level outcomes by randomly assigning individuals within each group to resource types according to the respective sub-budgets $\bar{\boldsymbol z}^{k(j)}_{r}$, using the shared learned outcome model $f$ to compute predicted individual outcomes. 
We then compute the mean predicted outcome $\mu^{k}_{r}(\bar{\boldsymbol z})$ for each group-resource pair and aggregate them into a global utility score:
\begin{equation}
\hat{y}(\bar{\boldsymbol z}) = \sum_{k \in K} \sum_{r \in R} \bar{\boldsymbol z}^{k}_{r} \cdot \mu^{k}_{r}(\bar{\boldsymbol z}).
\end{equation}
Instead of fitting and retraining a separate Gaussian Process (GP) over the high-dimensional meta-policy space, we reuse the learned outcome model $f$-which maps individual context to expected reward-as a simulation-based surrogate. This captures the functional relationship between allocation decisions and observed outcomes, enabling fast and scalable evaluation of candidate meta-policies $\bar{\boldsymbol z}$.

For any candidate $\bar{\boldsymbol z}$, we simulate assignments by probabilistically distributing resources within each subgroup in proportion to $\bar{z}^k_r$ and then use $f$ to predict individual outcomes. The resulting subgroup-level predictions are aggregated to estimate the overall utility $\hat{y}(\bar{\boldsymbol z})$.
This surrogate approach avoids the computational burden of GP inference in high dimensions while leveraging the contextual expressivity of $f$, which captures latent structure across individuals and resources. At each round $t_m$, to select the next candidate meta-policy, we adopt an UCB acquisition rule adapted to this simulation setting. Specifically, we estimate the posterior mean $\mu(\bar{\boldsymbol z})$ and empirical standard deviation $\sigma(\bar{\boldsymbol z})$ over multiple stochastic rollout simulations. The next meta-policy is then selected as the one maximizing the acquisition score:
\begin{equation}
    \bar{\boldsymbol z}(t_m)= \mathop{\mathrm{arg\,max}}\limits_{\bar{\boldsymbol z}\in \widetilde{\Delta}^{|K|\cdot|R|}}
    \bigl(\mu(\bar{\boldsymbol z}) + \beta_{t_m}\,\sigma(\bar{\boldsymbol z})\bigr)
\end{equation}
where $\beta_{t_m}$ is a time-dependent exploration parameter that balances exploitation of high-utility policies with exploration of uncertain regions of the meta-policy space. This acquisition strategy allows us to exploit the expressive power of $f$ while efficiently navigating the meta-policy space, eliminating the need to fit an explicit Gaussian Process. 
This approach scales effectively with dimensionality and adapts to contextual heterogeneity encoded in the population. As summarized in Algorithm~\eqref{alg:meta}, this meta-level optimization yields an optimal subgroup-level resource allocation policy $\bar{\boldsymbol z}^*$, which is then passed to the base-level assignment phase for individual-level decision making.

\textbf{Base-level Framework.}
Given the subgroup-level meta-allocation policy \(\bar{\boldsymbol{z}}^* = \{\bar{z}_{r}^k\}_{k \in K, r \in R}\) produced by the meta-level optimization, the base-level framework performs an individual-level contextual bandit search within each \((k, r)\) cell to identify the most promising recipients. This step refines the coarse-grained allocation \(\bar{z}_{r}^k\) by using a local UCB rule over individual contexts to balance exploitation of high-expected responders with exploration under uncertainty.
Let \(\mathcal{I}_k\) denote the set of eligible individuals in subgroup \(k\), and let \(f\) be the shared predictive model mapping individual context \(\mathbf{x}^i\) to expected outcome \(\hat{y}_{i,r} = f(\mathbf{x}^i)\). For each \((k,r)\) pair with \(\bar{z}_{r}^k > 0\), we define the target number of allocations as $n_{k,r} = \left\lfloor \bar{z}_{r}^k \cdot |\mathcal{I}_k| \right\rfloor.$
To allocate resource $r$ to the top $n_{k,r}$ individuals in $\mathcal{I}_k$, we compute UCB scores $G_{i,r} = \hat{y}_{i,r} + \beta u_{i,r}$, where $u_{i,r}$ denotes uncertainty (e.g., prediction variance), and $\beta$ balances exploration and exploitation. The top $n_{k,r}$ individuals by $G_{i,r}$ receive resource $r$. This yields an individualized policy that respects meta-level group constraints while exploiting within-group variation (Algorithm~\ref{alg:base}).

\begin{algorithm}[!ht]
\footnotesize
\caption{MetaCUB: Phase 1: Meta-level}
\label{alg:meta}
\textbf{Input:} Subgroups $K$, Resources $R$, Computational Budget $T_m$, Initial Policies $\{\bar{\boldsymbol z}^{(j)}\}_{j=1}^{n_0}$, Outcome model $f$\\
\textbf{Output:} $\bar{\boldsymbol z}^* = \mathop{\mathrm{arg\,max}}\limits_{(\bar{\boldsymbol z}, \hat{y}) \in \mathcal{D}_{T_m}} \hat{y}$
\begin{algorithmic}[1]
\STATE \textbf{Initialize:} Dataset $\mathcal{D}_0 \gets \emptyset$
\FOR{$j = 1$ to $n_0$}
    \STATE Simulate individual-level assignments using $\bar{\boldsymbol z}^{(j)}$
    \STATE $\hat{y}^{(j)} \gets \text{Evaluate}(f, \bar{\boldsymbol z}^{(j)})$ \COMMENT{Estimate outcome via $f$}
    \STATE $\mathcal{D}_0 \gets \mathcal{D}_0 \cup \{(\bar{\boldsymbol z}^{(j)}, \hat{y}^{(j)})\}$
\ENDFOR
\FOR{$t_m = n_0 + 1$ to $T_m$}
    \STATE Sample candidate set $\mathcal{S}_{t_m} \subset \widetilde{\Delta}^{|K| \cdot |R|}$
    \FOR{each $\bar{\boldsymbol z} \in \mathcal{S}_{t_m}$}
        \STATE Perform $B$ simulations under $\bar{\boldsymbol{z}}$ using $f$
        \STATE Estimate $\mu(\bar{\boldsymbol{z}})$ and $\sigma(\bar{\boldsymbol{z}})$
        \STATE Compute UCB score: $a(\bar{\boldsymbol z}) = \mu(\bar{\boldsymbol z}) + \beta_{t_m} \cdot \sigma(\bar{\boldsymbol z})$
    \ENDFOR
    \STATE Select best candidate: $\bar{\boldsymbol z}(t_m) \gets \mathop{\mathrm{arg\,max}}\limits_{\bar{\boldsymbol z} \in \mathcal{S}_{t_m}} a(\bar{\boldsymbol z})$
    \STATE Simulate assignments under $\bar{\boldsymbol z}(t_m)$
    \STATE $\hat{y}(t_m) \gets \text{Evaluate}(f, \bar{\boldsymbol{z}}(t_m))$
    \STATE $\mathcal{D}_{t_m} \gets \mathcal{D}_{t_m-1} \cup \{(\bar{\boldsymbol z}(t_m), \hat{y}(t_m))\}$
\ENDFOR
\end{algorithmic}
\end{algorithm}

\begin{algorithm}[!ht]
\footnotesize
\caption{MetaCUB: Phase 2: Base-level}
\label{alg:base}
\textbf{Inputs:} Meta-policy \(\bar{\boldsymbol{z}}^* = \{\bar{z}_{r}^k\}_{k \in K, r \in R}\); Outcome model \(f\); Individual sets \(\{\mathcal{I}_k\}_{k \in K}\)\\
\parbox[t]{\dimexpr\linewidth-1em}{\textbf{Output:} Individual-level allocation $\mathcal{A}$}

\begin{algorithmic}[1]
\STATE Initialize allocation set \(\mathcal{A} \gets \emptyset\)
\FOR{each subgroup \(k \in K\)}
    \FOR{each resource \(r \in R\)}
        \IF{\(\bar{z}_{r}^k > 0\)}
            \STATE Set allocation count: \(n_{k,r} \gets \left\lfloor \bar{z}_{r}^k \cdot |\mathcal{I}_k| \right\rfloor\)
            \FOR{each individual \(i \in \mathcal{I}_k\)}
                \STATE Compute predicted reward: \(\hat{y}_{i,r} \gets f(\mathbf{x}^i)\)
                \STATE Compute UCB score: \(G_{i,r} \gets \hat{y}_{i,r} + \beta \cdot u_{i,r}\)
            \ENDFOR
            \STATE Select top \(n_{k,r}\) individuals by \(G_{i,r}\): \(\mathcal{S}_{k,r}\)
            \STATE \(\mathcal{A} \gets \mathcal{A} \cup \{(i, r) \mid i \in \mathcal{S}_{k,r} \}\)
        \ENDIF
    \ENDFOR
\ENDFOR
\STATE \textbf{return} \(\mathcal{A}\)
\end{algorithmic}
\end{algorithm}

\subsection{Fairness Properties of MetaCUB}

MetaCUB’s bi-level design promotes equitable outcomes by decoupling global resource allocation (meta-level) from individual-level targeting (base-level). This structure mitigates group-level allocation disparities often amplified in flat contextual.

\begin{lemma}[Disparity Reduction]\label{lemma1}
Let $\mathcal{A}_{\texttt{MetaCUB}}$ denote the bi-level allocation under MetaCUB, and $\mathcal{A}_{\texttt{Flat}}$ denote the allocation from a one-level contextual bandit (e.g., LinUCB) that selects individual-resource pairs without subgroup constraints. Define disparity as the difference between the maximum and minimum average outcome across subgroups,
$
\text{Disparity}(\mathcal{A}) = \max_{k \in K} \bar{y}_k(\mathcal{A}) - \min_{k \in K} \bar{y}_k(\mathcal{A}),$ where $\bar{y}_k(\mathcal{A})$ is the mean outcome for subgroup $k$ under allocation $\mathcal{A}$. Then, under mild assumptions on exploration and model accuracy,
$
\text{Disparity}(\mathcal{A}_{\texttt{MetaCUB}}) \le \text{Disparity}(\mathcal{A}_{\texttt{Flat}}) - \delta(T_m, f),
$
for some $\delta(T_m, f) > 0$ that increases with meta-rounds $T_m$ and predictor fidelity $f$.
\end{lemma}

\begin{proof}[Sketch]
Flat contextual bandits optimize reward across individuals but can disproportionately favor dominant subgroups with higher estimated outcomes, leading to allocation imbalance. In contrast, MetaCUB first distributes resources across subgroups via meta-level optimization, ensuring broader coverage. Then, within each group, the base-level bandit targets high-benefit individuals. This structure bounds the inter-group disparity by ensuring minimum subgroup coverage and reducing waste through outcome-aware targeting. 
The fairness gap $\delta$ arises from this structure.\\
The full formal proof is provided in Appendix.
\end{proof}

\section{Experiments}\label{sec:experiments}
We evaluate the proposed bi-level delayed-feedback framework on two real-world datasets: the Educational Longitudinal Study (ELS)~\cite{els2002}, where resources represent financial aid packaging, and the JOBS randomized field experiment~\cite{jobs_nsw_psid}, where the resource is job training. Detailed dataset specifications appear in Appendix.
Our simulations cover varied experimental conditions, including delayed vs.\ immediate feedback, linear vs.\ nonlinear outcome mappings, different delay kernels, task types (regression for ELS, classification for JOBS), and resource dimensionality (multi-type in ELS, single-type in JOBS).

\begin{figure}[!ht]
    \centering
    \begin{subfigure}[t]{0.46\linewidth}
    \includegraphics[width=\linewidth]{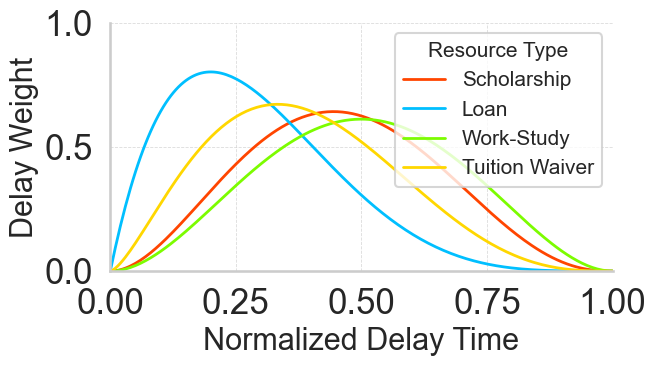}
        \caption{ELS Type-I kernels}
\label{fig:els_delay_kernels_type1}
    \end{subfigure}
    \hfill
    \begin{subfigure}[t]{0.46\linewidth}
        \includegraphics[width=\linewidth]{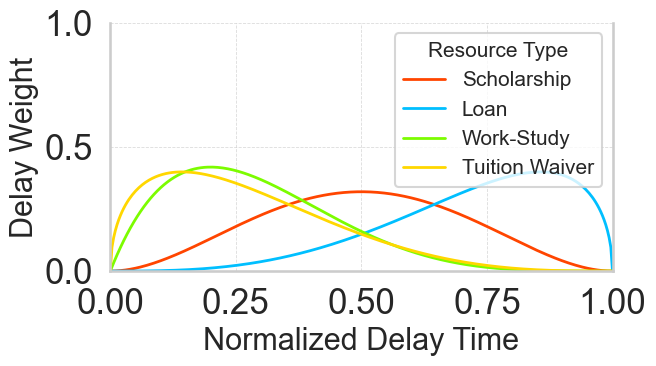}
\caption{ELS Type-II kernels}
        \label{fig:els_delay_kernels_type2}
    \end{subfigure}
    \begin{subfigure}[t]{0.46\linewidth}
        \includegraphics[width=\linewidth]{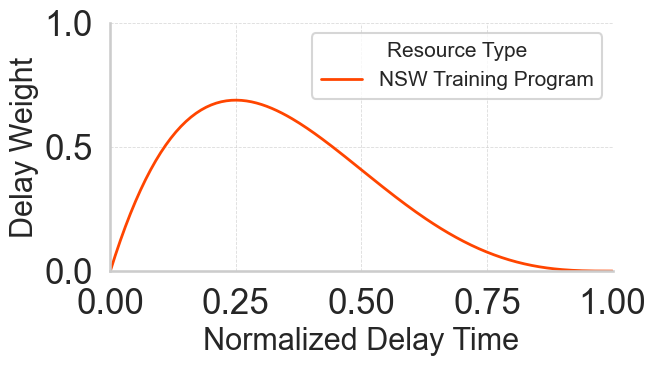}
\caption{JOBS Type-I kernel}
        \label{fig:jobs_delay_kernels_type1}
    \end{subfigure}
    \hfill
    \begin{subfigure}[t]{0.46\linewidth}
        \includegraphics[width=\linewidth]{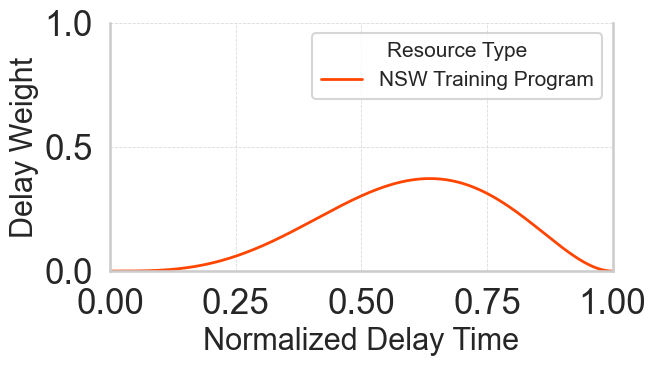}
\caption{JOBS Type-II kernel}
        \label{fig:jobs_delay_kernels_type2}
    \end{subfigure}
    \caption{Delay kernel distributions.}
\label{fig:delay_kernels_comparison}
\end{figure}

To systematically benchmark our method \emph{MetaCUB}, we compare its performance against a suite of baseline algorithms encompassing classical bandits, linear contextual approaches, combinatorial models, and adversarial formulations. \emph{UCB}~\cite{auer2002finite} treats each resource–recipient pair as an independent arm, ignoring context and group structure; \emph{LinUCB}~\cite{li2010contextual} incorporates individual contexts via linear regression but omits subgroup budgets; \emph{CUCB}~\cite{chen2016combinatorial} selects multiple arms per round yet suffers from limited feedback and scalability; \emph{EXP3}~\cite{auer2002nonstochastic} is robust to adversarial or delayed rewards but disregards stochastic structure and real‐world constraints; \emph{mEXP3}~\cite{tang2021bandit} explores over full allocation policies but incurs high variance in evaluation; and \emph{DUCB}~\cite{garivier2011upper} and \emph{SWUCB}~\cite{garivier2011upper} adapt to non‐stationarity via decay or sliding windows yet lack subgroup‐aware allocation mechanisms.  
A more detailed description of these baselines is provided in Table 2 in Appendix.

In both datasets, the number of base arms $K$ corresponds to racial subgroups. For ELS (GPA regression), we use ridge regression and a neural network for linear and nonlinear mappings, respectively. For JOBS (binary employment classification), we apply logistic regression and the same neural network to ensure consistent subgroup performance (e.g., 86\% overall accuracy on ELS: Asian=87\%, Black=85\%, Hispanic=86\%, White=86\%).
To simulate dynamic populations, individuals are grouped into fixed-length cohorts: 8 semesters for ELS and 12 months for JOBS. One cohort is active per round and replaced upon completion. Shaded bands in plots denote active cohort periods.
Feedback delay is modeled via two kernel types per dataset (Figure~\ref{fig:delay_kernels_comparison}): four resource-specific kernels in ELS and two variants for the single JOBS resource. These test our method’s robustness to heterogeneous, delayed rewards. To reflect real-world constraints, we impose stochastic cooldowns: after receiving a resource, individuals enter a cooldown sampled uniformly from ${1, 2, 3}$ rounds—adaptable to other settings.

All experiments are conducted with 20 independent random seeds to ensure robustness. Simulations are implemented in Python 3.11.5 using NumPy, scikit-learn, and BoTorch, and executed on an Apple M4 Pro (14-core CPU, 20-core GPU, 16-core Neural Engine, 24 GB RAM) running macOS 15.5. Reported performance metrics are averaged across all runs.

\section{Results}\label{sec:results}

The plots in Figure~\ref{fig:regret_album_ELS_I},~\ref{fig:regret_album_ELS_II},~\ref{fig:regret_album_JOBS_I}, and~\ref{fig:regret_album_JOBS_II} illustrate the cumulative regret trajectories of all baseline algorithms compared to our proposed method \emph{MetaCUB} across four experimental settings derived from both the ELS dataset and JOBS datasets with two distinct delay kernel configurations; lines show mean regret over 20 runs, and shaded regions indicate standard deviations. These settings vary along two dimensions: the nature of the outcome function (linear vs.\ nonlinear) and the presence or absence of delayed feedback (delayed vs. immediate). In all scenarios, \emph{MetaCUB} consistently achieves the lowest cumulative regret, demonstrating its superior ability to adaptively balance exploration and exploitation under both immediate and delayed reward settings. The performance gap is especially pronounced in the delayed-feedback environments, where conventional bandits like \emph{UCB} and \emph{EXP3} exhibit substantially higher regret due to their lack of temporal sensitivity. Algorithms such as \emph{DUCB} and \emph{SWUCB} show improved robustness under delay but still underperform relative to \emph{MetaCUB}, which leverages subgroup-level structure and kernelized delay modeling.

\begin{figure}[!ht]
    \centering
    \begin{subfigure}[t]{0.48\linewidth}
        \includegraphics[width=\linewidth]{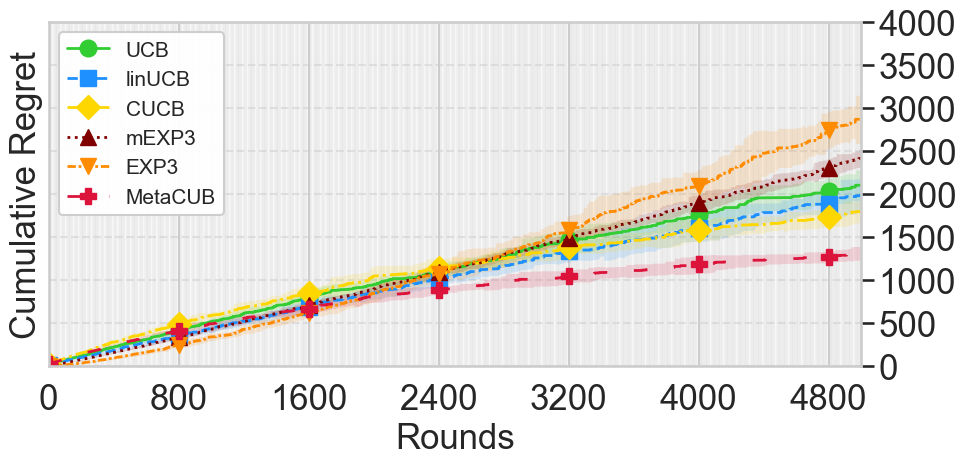}
        \caption{Immediate * linear}
        \label{fig:els_I_sub1}
    \end{subfigure}
    \hfill
    \begin{subfigure}[t]{0.48\linewidth}
        \includegraphics[width=\linewidth]{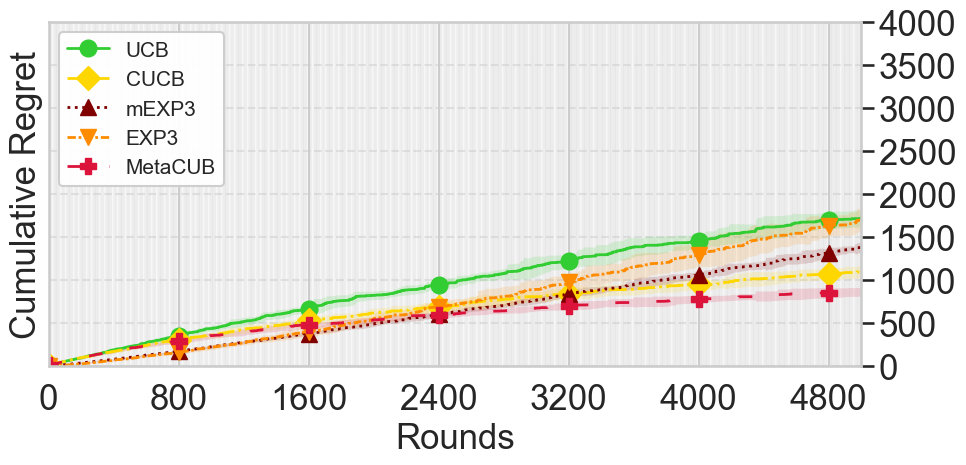}
        \caption{Immediate * non-linear }
        \label{fig:els_I_sub2}
    \end{subfigure}
    \begin{subfigure}[t]{0.48\linewidth}
        \includegraphics[width=\linewidth]{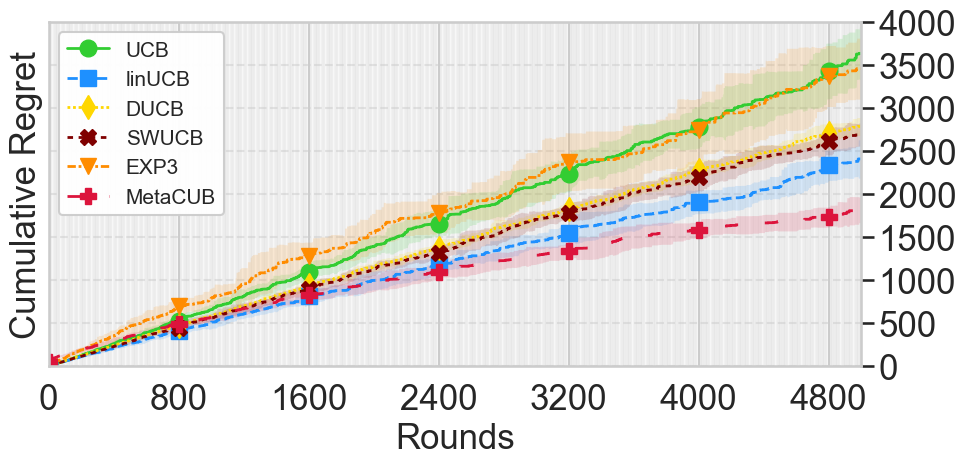}
        \caption{Delayed * linear}
        \label{fig:els_I_sub3}
    \end{subfigure}
    \hfill
    \begin{subfigure}[t]{0.48\linewidth}
        \includegraphics[width=\linewidth]{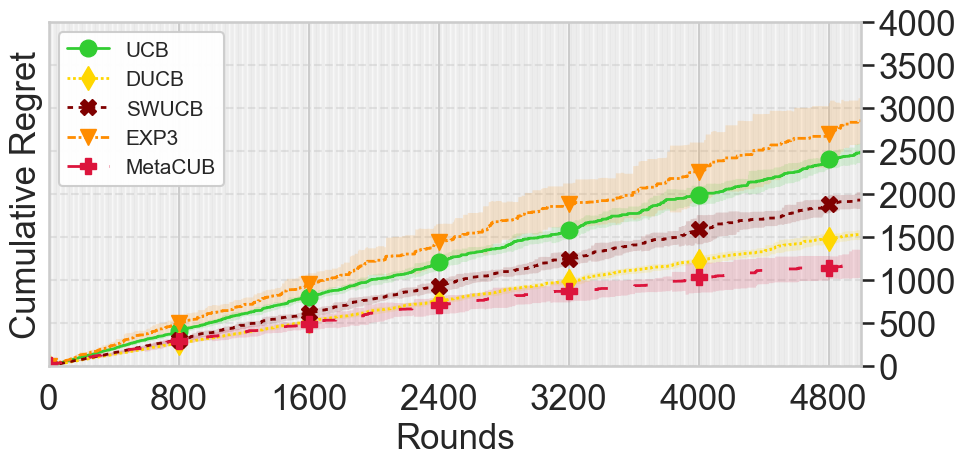}
        \caption{Delayed * non-linear}
        \label{fig:els_I_sub4}
    \end{subfigure}
    \caption{ELS, Delay Kernel Type-I: Cumulative regret.}
    \label{fig:regret_album_ELS_I}
\end{figure}

\begin{figure}[!ht]
    \centering
    \begin{subfigure}[t]{0.48\linewidth}
        \includegraphics[width=\linewidth]{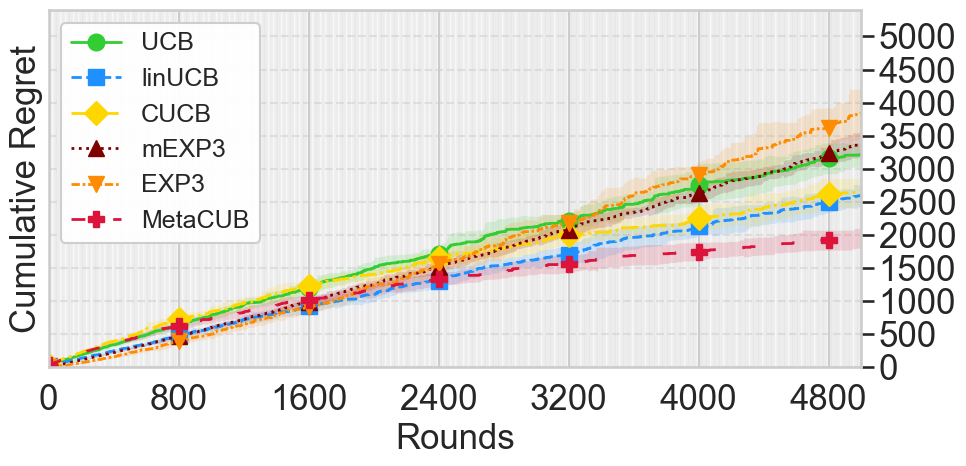}
        \caption{Immediate * linear }
        \label{fig:els_II_sub1}
    \end{subfigure}
    \hfill
    \begin{subfigure}[t]{0.48\linewidth}
        \includegraphics[width=\linewidth]{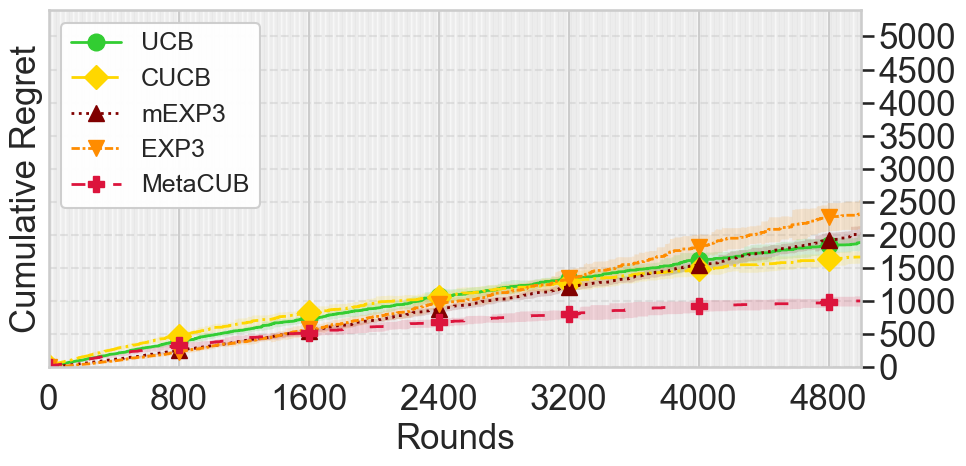}
        \caption{Immediate * non-linear}
        \label{fig:els_II_sub2}
    \end{subfigure}
    \begin{subfigure}[t]{0.48\linewidth}
        \includegraphics[width=\linewidth]{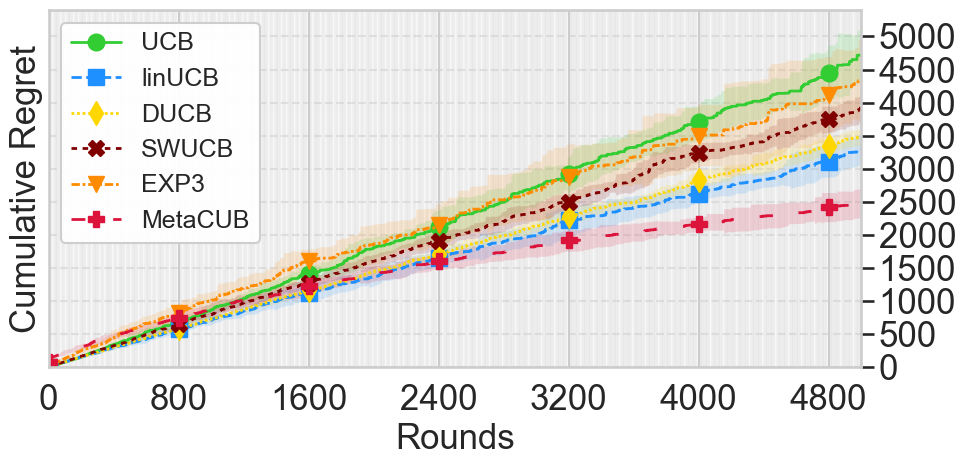}
        \caption{Delayed * linear}
        \label{fig:els_II_sub3}
    \end{subfigure}
    \hfill
    \begin{subfigure}[t]{0.48\linewidth}
        \includegraphics[width=\linewidth]{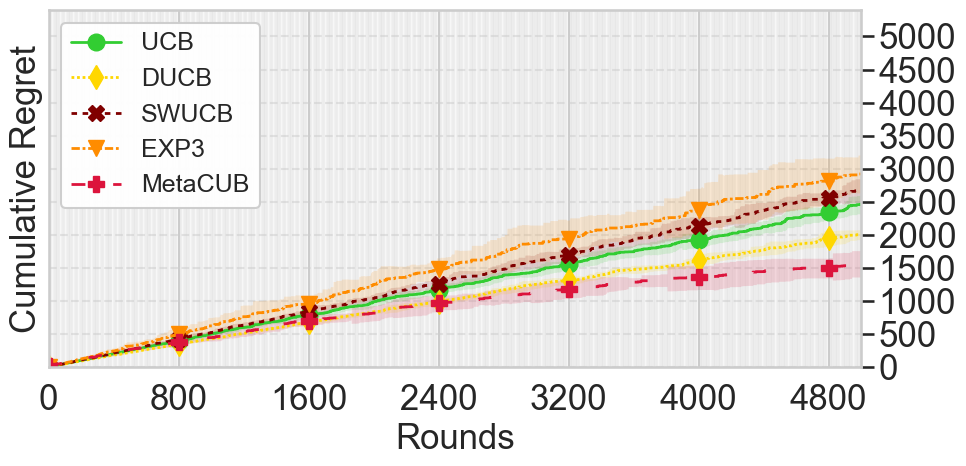}
        \caption{Delayed * non-linear}
        \label{fig:els_II_sub4}
    \end{subfigure}
    \caption{ELS, Delay Kernel Type-II: Cumulative regret.}
    \label{fig:regret_album_ELS_II}
\end{figure}

\begin{figure}[!ht]
    \centering
    \begin{subfigure}[t]{0.48\linewidth}
        \includegraphics[width=\linewidth]{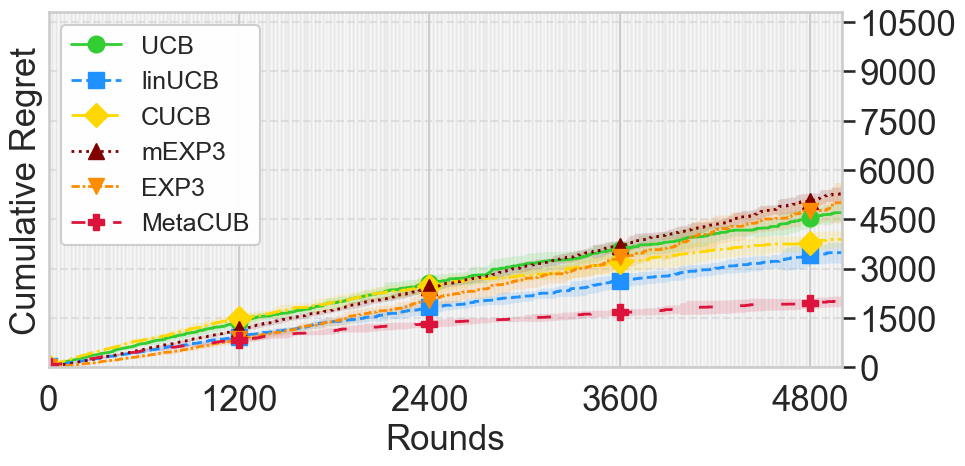}
        \caption{Immediate * linear}
        \label{fig:jobs_I_sub1}
    \end{subfigure}
    \hfill
    \begin{subfigure}[t]{0.48\linewidth}
        \includegraphics[width=\linewidth]{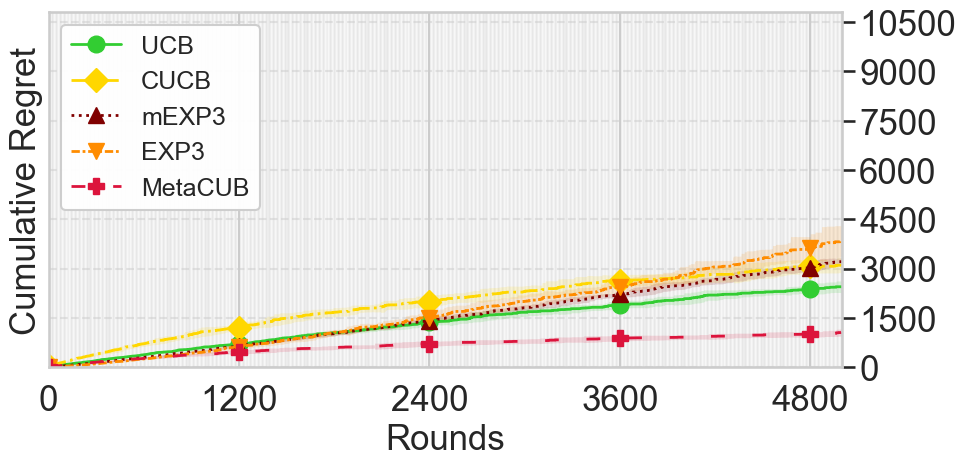}
        \caption{Immediate * non-linear }
        \label{fig:jobs_I_sub2}
    \end{subfigure}
    \begin{subfigure}[t]{0.48\linewidth}
        \includegraphics[width=\linewidth]{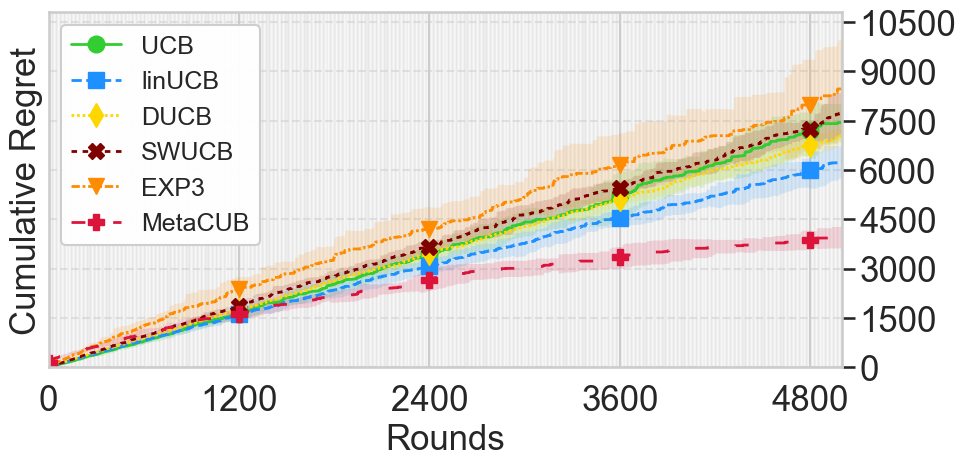}
        \caption{Delayed * linear}
        \label{fig:jobs_I_sub3}
    \end{subfigure}
    \hfill
    \begin{subfigure}[t]{0.48\linewidth}
        \includegraphics[width=\linewidth]{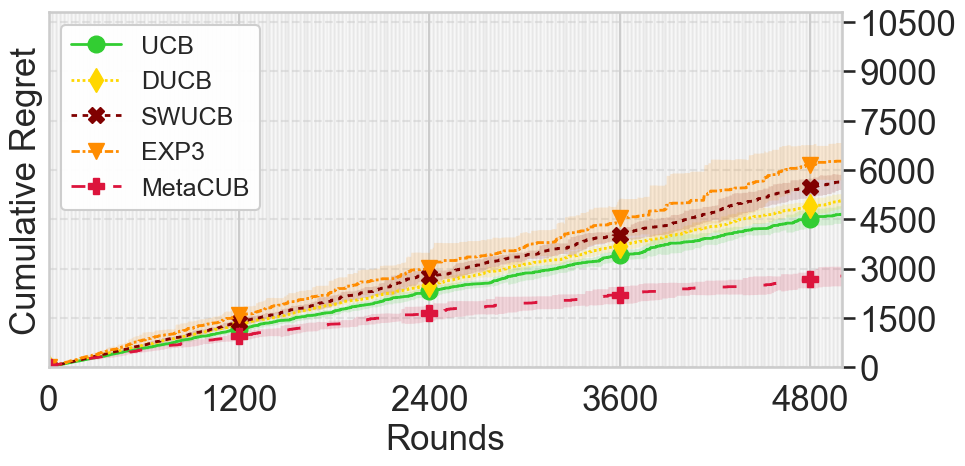}
        \caption{Delayed * non-linear}
        \label{fig:jobs_I_sub4}
    \end{subfigure}
    \caption{JOBS, Delay Kernel Type-I: Cumulative regret.}
    \label{fig:regret_album_JOBS_I}    
\end{figure}
\begin{figure}[!ht]
    \centering
    \begin{subfigure}[t]{0.48\linewidth}
        \includegraphics[width=\linewidth]{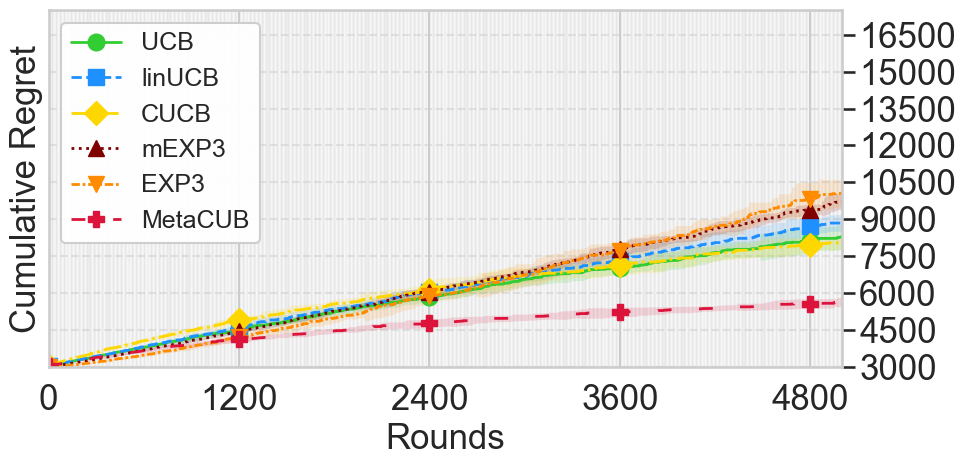}
        \caption{Immediate * linear}
        \label{fig:jobs_II_sub1}
    \end{subfigure}
    \hfill
    \begin{subfigure}[t]{0.48\linewidth}
        \includegraphics[width=\linewidth]{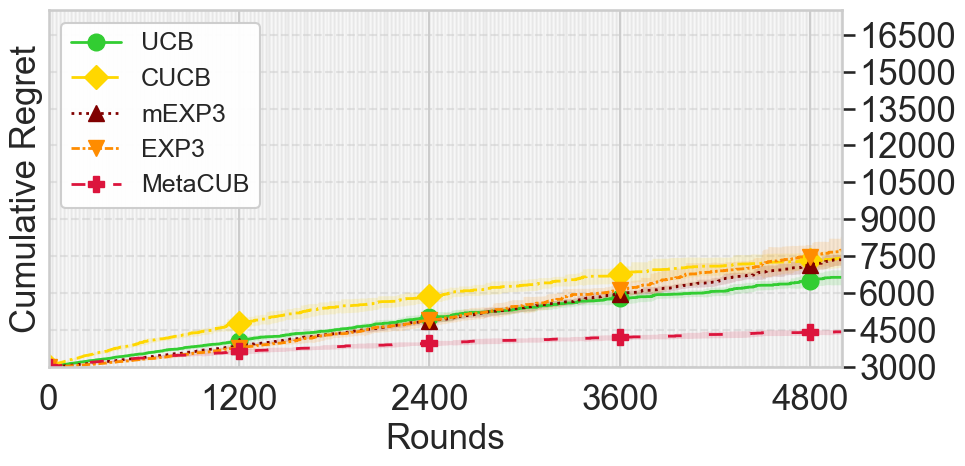}
        \caption{Immediate * non-linear}
        \label{fig:jobs_II_sub2}
    \end{subfigure}
    \begin{subfigure}[t]{0.48\linewidth}
        \includegraphics[width=\linewidth]{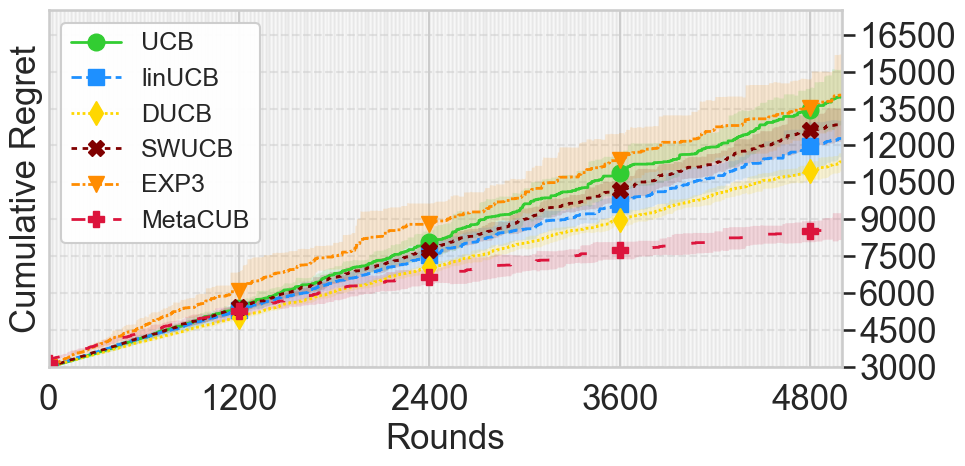}
        \caption{Delayed * linear }
        \label{fig:jobs_II_sub3}
    \end{subfigure}
    \hfill
    \begin{subfigure}[t]{0.48\linewidth}
        \includegraphics[width=\linewidth]{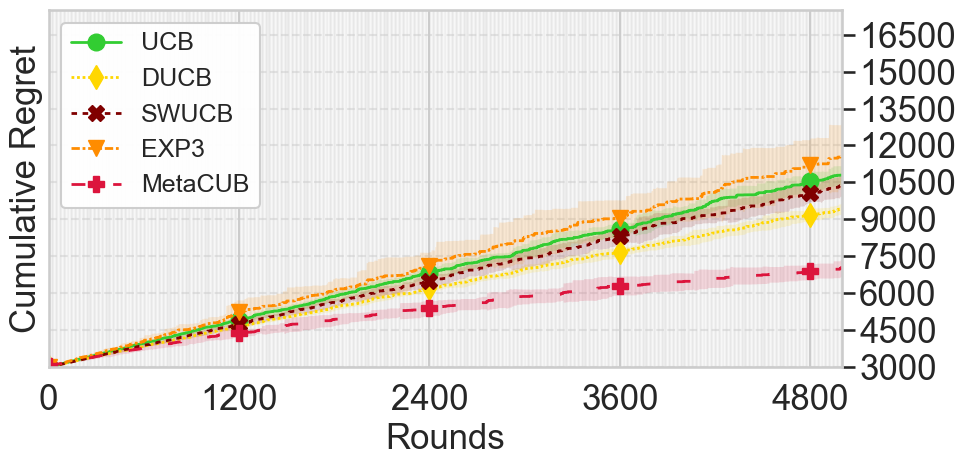}
        \caption{Delayed * non-linear}
        \label{fig:jobs_II_sub4}
    \end{subfigure}
    \caption{JOBS, Delay Kernel Type-II: Cumulative regret.}
    \label{fig:regret_album_JOBS_II}
\end{figure}

Moreover, the performance differences observed between both datasets experiments primarily arise from the distinct resource feedback delay kernels used in each scenario. Type-I kernels (Figure~\ref{fig:els_delay_kernels_type1}, and Figure~\ref{fig:jobs_delay_kernels_type1}) exhibit peaked, unimodal shapes with concentrated delay weights at early-to-mid normalized times, meaning that the majority of reward signals arrive relatively soon after allocation. This structure allows learning algorithms to receive informative feedback more rapidly, enabling faster adaptation and significantly lower cumulative regret—most notably for \emph{MetaCUB}, which leverages structured delay-awareness. In contrast, Type-II kernels (Figure~\ref{fig:els_delay_kernels_type2}, and Figure~\ref{fig:jobs_delay_kernels_type2}) are more flattened and dispersed, with broader support across the time axis. This results in temporally diluted reward signals and increased uncertainty, impeding the learning efficiency of delay-agnostic methods. Under these broader delays, baseline algorithms such as \emph{UCB}, \emph{EXP3}, and their variants accumulate regret more rapidly, particularly in the delayed linear case, while \emph{MetaCUB} remains consistently more resilient. These findings emphasize the practical importance of modeling heterogeneous and temporally diffuse feedback mechanisms when deploying learning-based allocation policies.
In both datasets, algorithms evaluated under nonlinear reward functions incur substantially lower regret than their linear counterparts. This demonstrates that expressive nonlinear representations, whether provided by flexible predictors or by kernelized policies, improve learning and decision making. Even classical methods such as \emph{UCB} and \emph{CUCB} benefit when the linearity assumption is relaxed, yielding more accurate reward estimates and more effective allocations. 
Building on this, we evaluate the fairness of allocation decisions by analyzing representational balance across subgroups. While Lemma~\ref{lemma1} supports fairness convergence, we complement it with empirical fairness ratios, defined as the proportion of selected individuals in each subgroup. As shown in Tables~\ref{tab:fairness_els_short}, and Tabels 4, and 5 in Appendix, \emph{MetaCUB} consistently achieves the most balanced subgroup coverage across both datasets and feedback regimes, validating its fairness-aware design.

\begin{table}[H]
\centering
\scriptsize
\renewcommand{\arraystretch}{1.15}
\setlength{\tabcolsep}{4.5pt}
\begin{tabular}{l|cc|cc|cc|cc}
\toprule
\textbf{Algs.} & \multicolumn{2}{c|}{\textbf{Asian}} & \multicolumn{2}{c|}{\textbf{White}} & \multicolumn{2}{c|}{\textbf{Black}} & \multicolumn{2}{c}{\textbf{Hispanic}} \\
\cmidrule{2-9}
 & \textbf{Imm} & \textbf{Del} & \textbf{Imm} & \textbf{Del} & \textbf{Imm} & \textbf{Del} & \textbf{Imm} & \textbf{Del} \\
\midrule
UCB       & 0.62 & 0.41 & 1.29 & 1.42 & 0.48 & 0.33 & 0.57 & 0.51 \\
CUCB      & 0.59 &  *   & 0.82 &   *  & 0.51 &  *   & 0.63 &  *   \\
EXP3      & 0.32 & 0.27 & 0.91 & 1.16 & 0.34 & 0.22 & 0.41 & 0.36 \\
mEXP3     & 0.45 &  *   & 1.27 &   *  & 0.36 &  *   & 0.28 &  *   \\
DUCB      &   *  & 0.22 &  *   & 1.05 &   *  & 0.57 &   *  & 0.52 \\
SWUCB     &   *  & 0.52 &  *   & 1.27 &   *  & 0.39 &   *  & 0.29 \\
\textbf{MetaCUB} & \textbf{0.84} & \textbf{1.02} & \textbf{1.03} & \textbf{0.96} & \textbf{1.02} & \textbf{1.00} & \textbf{0.98} & \textbf{0.97} \\
\bottomrule
\end{tabular}
\caption{ELS Allocation Fairness (full table in Appendix)}
\label{tab:fairness_els_short}
\end{table}

\section{Conclusion}
We propose \emph{MetaCUB}, a bi-level contextual bandit framework for fair and adaptive resource allocation under delayed feedback. \emph{MetaCUB} outperforms baselines in cumulative regret and achieves more balanced subgroup coverage, supported by both empirical results and theoretical guarantees.

\section{Appendix}\label{sec:appendix}
The appendix provides additional details on the algorithmic design, the proof of the disparity reduction lemma, fairness analysis and evaluation results, and dataset descriptions.
\subsection{Constrained Contextual MAB Algorithm}
\begin{algorithm}[!h]
\footnotesize
\caption{Constrained Contextual Bandit (UCB)}
\label{alg:ccmab}

\parbox[t]{\linewidth}{
\textbf{Inputs:} Resource budgets $\{b^{r}\}_{r \in R}$; Computational budget $T$; Outcome model $f$; Exploration schedule $\{\beta_t\}_{t=1}^T$
}

\parbox[t]{\linewidth}{
\textbf{Output:} Allocation sequence $\mathcal{A} = \{(i_t, r_t)\}_{t=1}^T$
}

\begin{algorithmic}[1]
\STATE Initialize: $\mathcal{A} \gets \emptyset$
\FOR{$t = 1$ to $T$}
    \STATE Observe pool $N$ and context $\{\mathbf{x}^i(t)\}_{i \in N}$
    \STATE $\mathcal{F}_t \gets \{(i, r) \in N \times R \;|\; b^r > 0\}$
    \IF{$\mathcal{F}_t = \emptyset$}
        \STATE \textbf{break}
    \ENDIF
    \FOR{each $(i, r) \in \mathcal{F}_t$}
        \STATE Predict outcome: $\hat{y}_{i,r} \gets f(\mathbf{x}^i(t))$
        \STATE Estimate uncertainty: $u_{i,r}(t) \gets \text{Uncertainty}(i, r, t)$
        \STATE Compute UCB score: $G_{i,r}(t) \gets \hat{y}_{i,r} + \beta_t \cdot u_{i,r}(t)$
    \ENDFOR
    \STATE Select assignment: $(i_t, r_t) \gets \mathop{\mathrm{arg\,max}}\limits_{(i, r) \in \mathcal{F}_t} G_{i,r}(t)$
    \STATE Allocate resource: $\mathcal{A} \gets \mathcal{A} \cup \{(i_t, r_t)\}$
    \STATE Update budget: $b^{r_t} \gets b^{r_t} - 1$
    \STATE Update model $f$ with feedback from $(\mathbf{x}^{i_t}(t), \hat{y}_{i_t, r_t}, r_t)$
\ENDFOR
\STATE \textbf{return} $\mathcal{A}$
\end{algorithmic}
\end{algorithm}
\footnotetext{This algorithm can be adapted to minimize cumulative regret instead of maximizing cumulative reward.}

\subsection{Proof of Lemma\ref{lemma1}} 
\begin{proof}
Let $K$ be the set of subgroups. For any allocation $\mathcal{A}$, define the disparity as
\begin{equation}
    \text{Disparity}(\mathcal{A}) := \max_{k \in K} \bar{y}_k(\mathcal{A}) - \min_{k \in K} \bar{y}_k(\mathcal{A})
\end{equation}
where $\bar{y}_k(\mathcal{A})$ denotes the average outcome for subgroup $k$ under allocation $\mathcal{A}$.
Assume, for contradiction, that
\begin{equation}
\label{eq:contra}
\text{Disparity}(\mathcal{A}_{\texttt{MetaCUB}}) > \text{Disparity}(\mathcal{A}_{\texttt{Flat}}) - \delta,\; \forall\, \delta > 0
\end{equation}
Let $T_m$ denote the number of meta-rounds and $f$ the fidelity of the outcome model.

\textit{Assumptions}:
\begin{itemize}
\item[(A1)] (Coverage) $\exists\, \gamma > 0$ such that for all $k \in K$, subgroup $k$ is selected in at least $\gamma T_m$ meta-rounds under $\mathcal{A}_{\texttt{MetaCUB}}$.
\item[(A2)] \parbox[t]{\dimexpr\linewidth-\labelwidth-\labelsep}{%
(Base-level convergence)\\ $\lim_{T_m \to \infty,\, f \to \infty} \bar{y}_k(\mathcal{A}_{\texttt{MetaCUB}}) = \bar{y}_k^*$, where $\bar{y}_k^*$ is the optimal within-group average outcome.%
}

\item[(A3)] \parbox[t]{\dimexpr\linewidth-\labelwidth-\labelsep}{%
(Flat imbalance) $\exists\, k^\dagger \in K$ such that the number of assignments to $k^\dagger$ under $\mathcal{A}_{\texttt{Flat}}$ is $o(T_m)$, implying $\bar{y}_{k^\dagger}(\mathcal{A}_{\texttt{Flat}}) < \bar{y}_{k^\dagger}^*$.%
}
\end{itemize}
Then, for sufficiently large $T_m$ and $f$,
\begin{equation}
    \bar{y}_k(\mathcal{A}_{\texttt{MetaCUB}}) \ge \bar{y}_k(\mathcal{A}_{\texttt{Flat}}), \quad \forall k \in K
\end{equation}
and
\begin{equation}
    \exists\, k^* \in K \text{ such that } \bar{y}_{k^*}(\mathcal{A}_{\texttt{MetaCUB}}) > \bar{y}_{k^*}(\mathcal{A}_{\texttt{Flat}})
\end{equation}
Hence,
\begin{equation}
    \min_k \bar{y}_k(\mathcal{A}_{\texttt{MetaCUB}}) \ge \min_k \bar{y}_k(\mathcal{A}_{\texttt{Flat}}) + \delta(T_m, f)
\end{equation}
for some $\delta(T_m, f) > 0$ increasing in $T_m$ and $f$.\\
Therefore,
\begin{align}
    \text{Disparity}(\mathcal{A}_{\texttt{MetaCUB}})=\\
    & \hspace{-1.6cm} \max_k \bar{y}_k(\mathcal{A}_{\texttt{MetaCUB}}) - \min_k \bar{y}_k(\mathcal{A}_{\texttt{MetaCUB}}) \\
    & \hspace{-3.1cm} \le \max_k \bar{y}_k(\mathcal{A}_{\texttt{Flat}}) - ( \min_k \bar{y}_k(\mathcal{A}_{\texttt{Flat}}) + \delta(T_m, f) ) \\
    &\hspace{-0.5cm}= \text{Disparity}(\mathcal{A}_{\texttt{Flat}}) - \delta(T_m, f)
\end{align}
contradicting \eqref{eq:contra}. The claim follows.
\end{proof}

\subsection{Baseline Methods and Characteristics}
\begin{table}[!h]
\centering
\scriptsize
\renewcommand{\arraystretch}{1.15}
\setlength{\tabcolsep}{4pt}
\caption{Baseline Algorithm Summary}
\label{tab:baselines}
\begin{tabular}{>{\centering\arraybackslash}m{2.6cm}|m{5.2cm}}
\toprule
\textbf{Algorithm} & 
 \multicolumn{1}{>{\centering\arraybackslash}m{5.2cm}}{\textbf{Description}} \\
\midrule
UCB & Classical bandit that treats each resource-recipient pair as an arm. Ignores context and group structure; selects based on reward optimism. \\
\midrule
LinUCB & A contextual bandit applying ridge regression to each arm’s context to estimate linear reward parameters and confidence bounds, capturing individual heterogeneity but ignoring group-level allocation.\\
\midrule
CUCB & Combinatorial semi-bandit that selects multiple group–resource arms per round; observes feedback only from selected pairs~\cite{chen2016combinatorial,wang2017improving}. \\
\midrule
EXP3 & Adversarial bandit using exponential weighting without stochastic reward assumptions; robust under partial and delayed feedback~\cite{auer2002nonstochastic}. \\
\midrule
mEXP3 & EXP3 over meta-arms, where each arm represents a full group–resource allocation policy~\cite{tang2021bandit}. \\
\midrule
DUCB & UCB variant with exponential decay ($\gamma$) on past rewards; adapts to non-stationary rewards~\cite{garivier2011upper,tang2021bandit}. \\
\midrule
SWUCB & UCB with sliding window of $\tau$ rounds; forgets older rewards to adapt to temporal changes~\cite{garivier2011upper}. \\
\bottomrule
\end{tabular}
\end{table}

\subsection{Datasets}

\textbf{Educational Longitudinal Study (ELS).}\label{subsec: ELS_data}
 The ELS dataset comprises individual student records, each endowed with a feature vector $\mathbf{x}^i\in\mathbb{R}^n$ representing a diverse set of demographic, academic, and contextual attributes. The features include continuous measures such as family income and cumulative grade point averages, ordinal scales for parental education and athletic involvement, binary indicators for gender and race categories, and categorical encodings for family composition and school urbanicity.  We denote the primary outcome as the cumulative GPA, treated as a continuous response.  Resource assignment follows a multi-resource protocol with four discrete resource types subject to resource-specific budget constraints $B^r$.\\
\textbf{JOBS Randomized Trial.}\label{subsec:JOBS_data}
The JOBS dataset merges the experimental National Supported Work (NSW) sample~\cite{lalonde1986evaluating} with a non-experimental control group from the Panel Study of Income Dynamics (PSID)~\cite{shalit2017estimating}, totaling $N=3{,}212$ individuals—297 treated and 2,915 controls. Each individual is represented by eight baseline covariates: age, education (years completed), Black/Hispanic race indicators, marital status, absence of a high-school diploma, and real earnings in 1974 and 1975. The primary continuous outcome is real earnings in 1978, with a derived binary indicator denoting employment status that year. This dataset is widely used in causal inference benchmarks to estimate job training effects on earnings and employment, offering rich pretreatment covariates and randomized treatment in its experimental subset~\cite{lalonde1986evaluating,shalit2017estimating}.
\subsubsection{Dataset Specifications Comparison.}
Table~\ref{tab:dataset_spec_comparison} summarizes and contrasts the two datasets used in our experiments.

\begin{table}[h]

\centering
\footnotesize
\begin{tabular}{lcc}
\toprule
\textbf{Specification} & \textbf{ELS} & \textbf{JOBS} \\
\midrule
No. Instances    & 1179          & 2935 \\
No. Features     & 40            & 8 \\
No. Resources    & 4             & 1 \\
Outcome type     & Continuous    & Binary \\
\bottomrule
\end{tabular}
\caption{JOBS and ELS Datasets Specification Comparison}
\label{tab:dataset_spec_comparison}
\end{table}

\subsection{Fairness Analysis}\label{sec:appendix-fairness}

Our setting focuses on high-stakes, dynamic resource allocation with delayed feedback, where ensuring equitable subgroup coverage is the primary institutional objective. In real-world socio-economic contexts—such as education, healthcare, workforce, and housing—fairness is typically operationalized at the subgroup level rather than through individual similarity metrics. Policy and regulatory frameworks consistently reflect this principle: Title VII enforcement applies the four-fifths (80\%) adverse-impact rule to ensure selection-rate parity across demographic groups; ECOA/Reg B fair-lending guidance addresses disparate impact across protected classes; HUD’s discriminatory-effects rule defines group-based disparate impact in housing; and public-health allocation frameworks (e.g., CDC/ACIP and National Academies) establish phase- or risk-group–based prioritization criteria. Similarly, education agencies and workforce programs set subgroup-specific targets or monitor disparities in subgroup outcomes to meet compliance and equity mandates.

Hence, the policy mandate we model is at the group level, and our guarantee matches that mandate. Notably, individual-level fairness does not guarantee group-level fairness. When one group is much larger, enforcing strict individual fairness can inadvertently widen inter-group disparities due to scale effects. This phenomenon is well documented in the fairness literature~\cite{fleisher2021s, castelnovo2022clarification}.
However, MetaCUB can incorporate individual-level fairness at the base level when properly defined. The only limitation is feasibility—if budgets or eligibility are too restrictive to meet both subgroup coverage and within-group fairness, the problem becomes infeasible or requires softened penalties. Finally, defining individual fairness under delayed outcomes is intrinsically harder: with stochastic delay kernels, “similar” individuals may realize impacts on different timelines, and rewards are only partially observed at decision time. This partial observability complicates any robust, similarity-based individual guarantee, further motivating a group-level fairness mandate for our application.

The following tables present the fairness results for all baseline algorithms and our proposed \textit{MetaCUB} across demographic subgroups in both datasets. Each value shows the average allocation disparity, the ratio of subgroup allocation rates to the overall mean, under immediate (\textbf{Imm}) and delayed (\textbf{Del}) feedback. Values closer to $1.0$ indicate more equitable allocations. As shown, \textit{MetaCUB} consistently attains near-parity across groups and feedback settings, confirming its robustness in maintaining fairness under both conditions.

\begin{table}[!h]
\centering
\scriptsize
\renewcommand{\arraystretch}{1.15}
\setlength{\tabcolsep}{3.8pt}
\begin{tabular}{l|cc|cc|cc|cc|cc}
\toprule
\textbf{Algs.} & \multicolumn{2}{c|}{\textbf{Asian}} & \multicolumn{2}{c|}{\textbf{White}} & \multicolumn{2}{c|}{\textbf{Black}} & \multicolumn{2}{c|}{\textbf{Hispanic}} & \multicolumn{2}{c}{\textbf{M-Race}} \\
\cmidrule{2-11}
 & \textbf{Imm} & \textbf{Del} & \textbf{Imm} & \textbf{Del} & \textbf{Imm} & \textbf{Del} & \textbf{Imm} & \textbf{Del} & \textbf{Imm} & \textbf{Del} \\
\midrule
UCB       & 0.62 & 0.41 & 1.29 & 1.42 & 0.48 & 0.33 & 0.57 & 0.51 & 0.44 & 0.31 \\
CUCB      & 0.59 &  *   & 0.82 &   *  & 0.51 &  *   & 0.63 &  *   & 0.71 &  *   \\
EXP3      & 0.32 & 0.27 & 0.91 & 1.16 & 0.34 & 0.22 & 0.41 & 0.36 & 0.58 & 0.44 \\
mEXP3     & 0.45 &  *   & 1.27 &   *  & 0.36 &  *   & 0.28 &  *   & 0.48 &  *   \\
DUCB      &   *  & 0.22 &  *   & 1.05 &   *  & 0.57 &   *  & 0.52 &   *  & 0.57 \\
SWUCB     &   *  & 0.52 &  *   & 1.27 &   *  & 0.39 &   *  & 0.29 &   *  & 0.25 \\
\textbf{MetaCUB} & \textbf{0.84} & \textbf{1.02} & \textbf{1.03} & \textbf{0.96} & \textbf{1.02} & \textbf{1.00} & \textbf{0.98} & \textbf{0.97} & \textbf{1.01} & \textbf{0.97} \\
\bottomrule
\end{tabular}
\caption{Full fairness scores for average allocation disparity on the ELS dataset.}
\label{tab:fairness_els}
\end{table}

\begin{table}[h]

\centering
\scriptsize
\renewcommand{\arraystretch}{1.15}
\setlength{\tabcolsep}{6pt}
\begin{tabular}{l|cc|cc}
\toprule
\textbf{Algs.} & \multicolumn{2}{c|}{\textbf{Black}} & \multicolumn{2}{c}{\textbf{Hispanic}} \\
\cmidrule{2-5}
 & \textbf{Imm} & \textbf{Del} & \textbf{Imm} & \textbf{Del} \\
\midrule
UCB       & 0.76 & 0.54 & 0.62 & 0.86 \\
CUCB      & 1.18 &   *  & 0.81 &   *  \\
EXP3      & 1.20 & 0.74 & 0.79 & 0.66 \\
mEXP3     & 1.12 &   *  & 0.87 &   *  \\
DUCB      &   *  & 1.29 &   *  & 0.79 \\
SWUCB     &   *  & 1.30 &   *  & 0.75 \\
\textbf{MetaCUB} & \textbf{1.03} & \textbf{0.98} & \textbf{0.99} & \textbf{0.98} \\
\bottomrule
\end{tabular}
\caption{Full fairness scores for average allocation disparity on the JOBS dataset.}
\label{tab:fairness_jobs}
\end{table}

\bibliography{aaai2026}

@misc{els2002,
  author = {National Center for Education Statistics},
  title = {Education Longitudinal Study of 2002 (ELS:2002)},
  year = {2025},
  url = {https://nces.ed.gov/surveys/els2002/},
  note = {Accessed: July 15, 2025}
}

@misc{jobs_nsw_psid,
  author = {Dehejia, Rajeev and Wahba, Sadek},
  title = {NSW/PSID Job Training Data (Jobs dataset)},
  year = {2025},
  url = {https://users.nber.org/~rdehejia/nswdata2.html},
  note = {Accessed: July 15, 2025}
}

@article{lalonde1986evaluating,
  title={Evaluating the econometric evaluations of training programs with experimental data},
  author={LaLonde, Robert J},
  journal={The American economic review},
  pages={604--620},
  year={1986},
  publisher={JSTOR}
}

@inproceedings{shalit2017estimating,
  title={Estimating individual treatment effect: generalization bounds and algorithms},
  author={Shalit, Uri and Johansson, Fredrik D and Sontag, David},
  booktitle={International conference on machine learning},
  pages={3076--3085},
  year={2017},
  organization={PMLR}
}

@article{lane2017equity,
  title={Equity in healthcare resource allocation decision making: a systematic review},
  author={Lane, Haylee and Sarkies, Mitchell and Martin, Jennifer and Haines, Terry},
  journal={Social science \& medicine},
  volume={175},
  pages={11--27},
  year={2017},
  publisher={Elsevier}
}

@article{aktacs2007decision,
  title={A decision support system to improve the efficiency of resource allocation in healthcare management},
  author={Akta{\c{s}}, Emel and {\"U}lengin, F{\"u}sun and {\c{S}}ahin, {\c{S}}ule {\"O}nsel},
  journal={Socio-Economic Planning Sciences},
  volume={41},
  number={2},
  pages={130--146},
  year={2007},
  publisher={Elsevier}
}

@article{daniels2016resource,
  title={Resource allocation and priority setting},
  author={Daniels, Norman and del Pilar Guzm{\'a}n Urrea, Mar{\'\i}a and Rentmeester, Christy A and Kotchian, Sarah Ann and Fontaine, Sherry and Hern{\'a}ndez-Aguado, Ildefonso and Lumbreras, Blanca and Blacksher, Erika and Goold, Susan D and G{\'o}mez, M In{\'e}s and others},
  journal={Public health ethics: Cases spanning the globe},
  pages={61--94},
  year={2016},
  publisher={Springer International Publishing Cham}
}

@article{su2019resource,
  title={Resource allocation for network slicing in 5G telecommunication networks: A survey of principles and models},
  author={Su, Ruoyu and Zhang, Dengyin and Venkatesan, Ramachandran and Gong, Zijun and Li, Cheng and Ding, Fei and Jiang, Fan and Zhu, Ziyang},
  journal={IEEE Network},
  volume={33},
  number={6},
  pages={172--179},
  year={2019},
  publisher={IEEE}
}

@article{hui2002resource,
  title={Resource allocation for broadband networks},
  author={Hui, Joseph Y},
  journal={IEEE Journal on selected areas in communications},
  volume={6},
  number={9},
  pages={1598--1608},
  year={2002},
  publisher={IEEE}
}

@inproceedings{gibney1998dynamic,
  title={Dynamic resource allocation by market-based routing in telecommunications networks},
  author={Gibney, MA and Jennings, Nicholas R},
  booktitle={International Workshop on Intelligent Agents for Telecommunication Applications},
  pages={102--117},
  year={1998},
  organization={Springer}
}

@article{monk1981toward,
  title={Toward a multilevel perspective on the allocation of educational resources},
  author={Monk, David H},
  journal={Review of Educational Research},
  volume={51},
  number={2},
  pages={215--236},
  year={1981},
  publisher={Sage Publications Sage CA: Thousand Oaks, CA}
}

@article{liefner2003funding,
  title={Funding, resource allocation, and performance in higher education systems},
  author={Liefner, Ingo},
  journal={Higher education},
  volume={46},
  number={4},
  pages={469--489},
  year={2003},
  publisher={Springer}
}

@book{massy1996resource,
  title={Resource allocation in higher education},
  author={Massy, William F},
  year={1996},
  publisher={University of Michigan Press}
}

@article{nguyen2014computational,
  title={Computational complexity and approximability of social welfare optimization in multiagent resource allocation},
  author={Nguyen, Nhan-Tam and Nguyen, Trung Thanh and Roos, Magnus and Rothe, J{\"o}rg},
  journal={Autonomous agents and multi-agent systems},
  volume={28},
  number={2},
  pages={256--289},
  year={2014},
  publisher={Springer}
}

@inproceedings{roos2010complexity,
  title={Complexity of social welfare optimization in multiagent resource allocation},
  author={Roos, Magnus and Rothe, J{\"o}rg},
  booktitle={Proceedings of the 9th International Conference on Autonomous Agents and Multiagent Systems: volume 1-Volume 1},
  pages={641--648},
  year={2010}
}

@article{hegazy1999optimization,
  title={Optimization of resource allocation and leveling using genetic algorithms},
  author={Hegazy, Tarek},
  journal={Journal of construction engineering and management},
  volume={125},
  number={3},
  pages={167--175},
  year={1999},
  publisher={American Society of Civil Engineers}
}

@article{gong2012efficient,
  title={An efficient resource allocation scheme using particle swarm optimization},
  author={Gong, Yue-Jiao and Zhang, Jun and Chung, Henry Shu-Hung and Chen, Wei-Neng and Zhan, Zhi-Hui and Li, Yun and Shi, Yu-Hui},
  journal={IEEE Transactions on Evolutionary Computation},
  volume={16},
  number={6},
  pages={801--816},
  year={2012},
  publisher={IEEE}
}

@article{obermeyer2019dissecting,
  title={Dissecting racial bias in an algorithm used to manage the health of populations},
  author={Obermeyer, Ziad and Powers, Brian and Vogeli, Christine and Mullainathan, Sendhil},
  journal={Science},
  volume={366},
  number={6464},
  pages={447--453},
  year={2019},
  publisher={American Association for the Advancement of Science}
}

@article{chouldechova2017fair,
  title={Fair prediction with disparate impact: A study of bias in recidivism prediction instruments},
  author={Chouldechova, Alexandra},
  journal={Big data},
  volume={5},
  number={2},
  pages={153--163},
  year={2017},
  publisher={Mary Ann Liebert, Inc. 140 Huguenot Street, 3rd Floor New Rochelle, NY 10801 USA}
}

@article{tang2021bandit,
  title={Bandit learning with delayed impact of actions},
  author={Tang, Wei and Ho, Chien-Ju and Liu, Yang},
  journal={Advances in Neural Information Processing Systems},
  volume={34},
  pages={26804--26817},
  year={2021}
}

@inproceedings{zou2019reinforcement,
  title={Reinforcement learning to optimize long-term user engagement in recommender systems},
  author={Zou, Lixin and Xia, Long and Ding, Zhuoye and Song, Jiaxing and Liu, Weidong and Yin, Dawei},
  booktitle={Proceedings of the 25th ACM SIGKDD international conference on knowledge discovery \& data mining},
  pages={2810--2818},
  year={2019}
}

@inproceedings{wang2022surrogate,
  title={Surrogate for long-term user experience in recommender systems},
  author={Wang, Yuyan and Sharma, Mohit and Xu, Can and Badam, Sriraj and Sun, Qian and Richardson, Lee and Chung, Lisa and Chi, Ed H and Chen, Minmin},
  booktitle={Proceedings of the 28th ACM SIGKDD conference on knowledge discovery and data mining},
  pages={4100--4109},
  year={2022}
}

@article{hanna2020mortality,
  title={Mortality due to cancer treatment delay: systematic review and meta-analysis},
  author={Hanna, Timothy P and King, Will D and Thibodeau, Stephane and Jalink, Matthew and Paulin, Gregory A and Harvey-Jones, Elizabeth and O’Sullivan, Dylan E and Booth, Christopher M and Sullivan, Richard and Aggarwal, Ajay},
  journal={bmj},
  volume={371},
  year={2020},
  publisher={British Medical Journal Publishing Group}
}

@inproceedings{grover2018best,
  title={Best arm identification in multi-armed bandits with delayed feedback},
  author={Grover, Aditya and Markov, Todor and Attia, Peter and Jin, Norman and Perkins, Nicolas and Cheong, Bryan and Chen, Michael and Yang, Zi and Harris, Stephen and Chueh, William and others},
  booktitle={International conference on artificial intelligence and statistics},
  pages={833--842},
  year={2018},
  organization={PMLR}
}

@inproceedings{gyorgy2021adapting,
  title={Adapting to delays and data in adversarial multi-armed bandits},
  author={Gyorgy, Andras and Joulani, Pooria},
  booktitle={International Conference on Machine Learning},
  pages={3988--3997},
  year={2021},
  organization={PMLR}
}

@inproceedings{joulani2013online,
  title={Online learning under delayed feedback},
  author={Joulani, Pooria and Gyorgy, Andras and Szepesv{\'a}ri, Csaba},
  booktitle={International conference on machine learning},
  pages={1453--1461},
  year={2013},
  organization={PMLR}
}

@article{almalki2022effect,
  title={The Effect of Immediate and Delayed Feedback in Virtual Classes on Mathematics Students' Higher Order Thinking Skills.},
  author={Almalki, Abdulaziz Derwesh A and Mohammed, Abdellah Ibrahim},
  journal={Journal of Positive School Psychology},
  volume={6},
  number={6},
  year={2022}
}

@article{yanovski2014long,
  title={Long-term drug treatment for obesity: a systematic and clinical review},
  author={Yanovski, Susan Z and Yanovski, Jack A},
  journal={Jama},
  volume={311},
  number={1},
  pages={74--86},
  year={2014},
  publisher={American Medical Association}
}

@article{barnett1995long,
  title={Long-term effects of early childhood programs on cognitive and school outcomes},
  author={Barnett, W Steven},
  journal={The future of children},
  pages={25--50},
  year={1995},
  publisher={JSTOR}
}

@article{edmondson2019co,
  title={The co-evolution of policy mixes and socio-technical systems: Towards a conceptual framework of policy mix feedback in sustainability transitions},
  author={Edmondson, Duncan L and Kern, Florian and Rogge, Karoline S},
  journal={Research Policy},
  volume={48},
  number={10},
  pages={103555},
  year={2019},
  publisher={Elsevier}
}

@article{kuang2023posterior,
  title={Posterior sampling with delayed feedback for reinforcement learning with linear function approximation},
  author={Kuang, Nikki Lijing and Yin, Ming and Wang, Mengdi and Wang, Yu-Xiang and Ma, Yian},
  journal={Advances in neural information processing systems},
  volume={36},
  pages={6782--6824},
  year={2023}
}

@article{yin2023long,
  title={Long-term fairness with unknown dynamics},
  author={Yin, Tongxin and Raab, Reilly and Liu, Mingyan and Liu, Yang},
  journal={Advances in Neural Information Processing Systems},
  volume={36},
  pages={55110--55139},
  year={2023}
}

@inproceedings{shi2023statistical,
  title={Statistical inference on multi-armed bandits with delayed feedback},
  author={Shi, Lei and Wang, Jingshen and Wu, Tianhao},
  booktitle={International Conference on Machine Learning},
  pages={31328--31352},
  year={2023},
  organization={PMLR}
}

@inproceedings{lancewicki2021stochastic,
  title={Stochastic multi-armed bandits with unrestricted delay distributions},
  author={Lancewicki, Tal and Segal, Shahar and Koren, Tomer and Mansour, Yishay},
  booktitle={International Conference on Machine Learning},
  pages={5969--5978},
  year={2021},
  organization={PMLR}
}

@inproceedings{zuo2021combinatorial,
  title={Combinatorial multi-armed bandits for resource allocation},
  author={Zuo, Jinhang and Joe-Wong, Carlee},
  booktitle={2021 55th Annual Conference on Information Sciences and Systems (CISS)},
  pages={1--4},
  year={2021},
  organization={IEEE}
}

@article{wang2019distributed,
  title={Distributed bandit learning: Near-optimal regret with efficient communication},
  author={Wang, Yuanhao and Hu, Jiachen and Chen, Xiaoyu and Wang, Liwei},
  journal={arXiv preprint arXiv:1904.06309},
  year={2019}
}

@inproceedings{li2022efficient,
  title={Efficient resource allocation with fairness constraints in restless multi-armed bandits},
  author={Li, Dexun and Varakantham, Pradeep},
  booktitle={Uncertainty in Artificial Intelligence},
  pages={1158--1167},
  year={2022},
  organization={PMLR}
}

@article{patil2021achieving,
  title={Achieving fairness in the stochastic multi-armed bandit problem},
  author={Patil, Vishakha and Ghalme, Ganesh and Nair, Vineet and Narahari, Yadati},
  journal={Journal of Machine Learning Research},
  volume={22},
  number={174},
  pages={1--31},
  year={2021}
}

@article{kuleshov2014algorithms,
  title={Algorithms for multi-armed bandit problems},
  author={Kuleshov, Volodymyr and Precup, Doina},
  journal={arXiv preprint arXiv:1402.6028},
  year={2014}
}

@inproceedings{agrawal2012analysis,
  title={Analysis of thompson sampling for the multi-armed bandit problem},
  author={Agrawal, Shipra and Goyal, Navin},
  booktitle={Conference on learning theory},
  pages={39--1},
  year={2012},
  organization={JMLR Workshop and Conference Proceedings}
}

@article{huo2017risk,
  title={Risk-aware multi-armed bandit problem with application to portfolio selection},
  author={Huo, Xiaoguang and Fu, Feng},
  journal={Royal Society open science},
  volume={4},
  number={11},
  pages={171377},
  year={2017},
  publisher={The Royal Society Publishing}
}

@inproceedings{bouneffouf2020survey,
  title={Survey on applications of multi-armed and contextual bandits},
  author={Bouneffouf, Djallel and Rish, Irina and Aggarwal, Charu},
  booktitle={2020 IEEE congress on evolutionary computation (CEC)},
  pages={1--8},
  year={2020},
  organization={IEEE}
}

@article{biswas2021learn,
  title={Learn to intervene: An adaptive learning policy for restless bandits in application to preventive healthcare},
  author={Biswas, Arpita and Aggarwal, Gaurav and Varakantham, Pradeep and Tambe, Milind},
  journal={arXiv preprint arXiv:2105.07965},
  year={2021}
}

@article{zhou2015survey,
  title={A survey on contextual multi-armed bandits},
  author={Zhou, Li},
  journal={arXiv preprint arXiv:1508.03326},
  year={2015}
}

@inproceedings{kaufmann2012bayesian,
  title={On Bayesian upper confidence bounds for bandit problems},
  author={Kaufmann, Emilie and Capp{\'e}, Olivier and Garivier, Aur{\'e}lien},
  booktitle={Artificial intelligence and statistics},
  pages={592--600},
  year={2012},
  organization={PMLR}
}

@article{cui2019multi,
  title={Multi-agent reinforcement learning-based resource allocation for UAV networks},
  author={Cui, Jingjing and Liu, Yuanwei and Nallanathan, Arumugam},
  journal={IEEE Transactions on Wireless Communications},
  volume={19},
  number={2},
  pages={729--743},
  year={2019},
  publisher={IEEE}
}

@inproceedings{agrawal2013thompson,
  title={Thompson sampling for contextual bandits with linear payoffs},
  author={Agrawal, Shipra and Goyal, Navin},
  booktitle={International conference on machine learning},
  pages={127--135},
  year={2013},
  organization={PMLR}
}

@article{xu2020collaborative,
  title={Collaborative multi-agent multi-armed bandit learning for small-cell caching},
  author={Xu, Xianzhe and Tao, Meixia and Shen, Cong},
  journal={IEEE Transactions on Wireless Communications},
  volume={19},
  number={4},
  pages={2570--2585},
  year={2020},
  publisher={IEEE}
}

@inproceedings{gael2020stochastic,
  title={Stochastic bandits with arm-dependent delays},
  author={Gael, Manegueu Anne and Vernade, Claire and Carpentier, Alexandra and Valko, Michal},
  booktitle={International Conference on Machine Learning},
  pages={3348--3356},
  year={2020},
  organization={PMLR}
}

@article{vernade2017stochastic,
  title={Stochastic bandit models for delayed conversions},
  author={Vernade, Claire and Capp{\'e}, Olivier and Perchet, Vianney},
  journal={arXiv preprint arXiv:1706.09186},
  year={2017}
}

@inproceedings{steiger2022learning,
  title={Learning from delayed semi-bandit feedback under strong fairness guarantees},
  author={Steiger, Juaren and Li, Bin and Lu, Ning},
  booktitle={IEEE INFOCOM 2022-IEEE Conference on Computer Communications},
  pages={1379--1388},
  year={2022},
  organization={IEEE}
}

@inproceedings{erez2024regret,
  title={Regret Guarantees for Adversarial Contextual Bandits with Delayed Feedback},
  author={Erez, Liad and Levy, Orin and Mansour, Yishay},
  booktitle={Seventeenth European Workshop on Reinforcement Learning},
  year={2024}
}

@article{badanidiyuru2018bandits,
  title={Bandits with knapsacks},
  author={Badanidiyuru, Ashwinkumar and Kleinberg, Robert and Slivkins, Aleksandrs},
  journal={Journal of the ACM (JACM)},
  volume={65},
  number={3},
  pages={1--55},
  year={2018},
  publisher={ACM New York, NY, USA}
}

@inproceedings{tran2012knapsack,
  title={Knapsack based optimal policies for budget--limited multi--armed bandits},
  author={Tran-Thanh, Long and Chapman, Archie and Rogers, Alex and Jennings, Nicholas},
  booktitle={Proceedings of the AAAI Conference on Artificial Intelligence},
  volume={26},
  pages={1134--1140},
  year={2012}
}

@article{burnetas2025optimal,
  title={Optimal data driven resource allocation under multi-armed bandit observations},
  author={Burnetas, Apostolos N and Kanavetas, Odysseas and Katehakis, Michael N},
  journal={Annals of Operations Research},
  pages={1--28},
  year={2025},
  publisher={Springer}
}

@inproceedings{chen2020fair,
  title={Fair contextual multi-armed bandits: Theory and experiments},
  author={Chen, Yifang and Cuellar, Alex and Luo, Haipeng and Modi, Jignesh and Nemlekar, Heramb and Nikolaidis, Stefanos},
  booktitle={Conference on Uncertainty in Artificial Intelligence},
  pages={181--190},
  year={2020},
  organization={PMLR}
}

@inproceedings{claure2020multi,
  title={Multi-armed bandits with fairness constraints for distributing resources to human teammates},
  author={Claure, Houston and Chen, Yifang and Modi, Jignesh and Jung, Malte and Nikolaidis, Stefanos},
  booktitle={Proceedings of the 2020 ACM/IEEE International Conference on Human-Robot Interaction},
  pages={299--308},
  year={2020}
}

@article{liu2012learning,
  title={Learning in a changing world: Restless multiarmed bandit with unknown dynamics},
  author={Liu, Haoyang and Liu, Keqin and Zhao, Qing},
  journal={IEEE Transactions on Information Theory},
  volume={59},
  number={3},
  pages={1902--1916},
  year={2012},
  publisher={IEEE}
}

@article{chen2022uncertainty,
  title={Uncertainty-of-information scheduling: A restless multiarmed bandit framework},
  author={Chen, Gongpu and Liew, Soung Chang and Shao, Yulin},
  journal={IEEE Transactions on Information Theory},
  volume={68},
  number={9},
  pages={6151--6173},
  year={2022},
  publisher={IEEE}
}

@inproceedings{mate2022field,
  title={Field study in deploying restless multi-armed bandits: Assisting non-profits in improving maternal and child health},
  author={Mate, Aditya and Madaan, Lovish and Taneja, Aparna and Madhiwalla, Neha and Verma, Shresth and Singh, Gargi and Hegde, Aparna and Varakantham, Pradeep and Tambe, Milind},
  booktitle={Proceedings of the AAAI Conference on Artificial Intelligence},
  number={11},
  pages={12017--12025},
  year={2022}
}

@article{chen2016combinatorial,
  title={Combinatorial multi-armed bandit and its extension to probabilistically triggered arms},
  author={Chen, Wei and Wang, Yajun and Yuan, Yang and Wang, Qinshi},
  journal={Journal of Machine Learning Research},
  volume={17},
  number={50},
  pages={1--33},
  year={2016}
}

@article{wang2017improving,
  title={Improving regret bounds for combinatorial semi-bandits with probabilistically triggered arms and its applications},
  author={Wang, Qinshi and Chen, Wei},
  journal={Advances in Neural Information Processing Systems},
  volume={30},
  year={2017}
}

@article{auer2002nonstochastic,
  title={The nonstochastic multiarmed bandit problem},
  author={Auer, Peter and Cesa-Bianchi, Nicolo and Freund, Yoav and Schapire, Robert E},
  journal={SIAM journal on computing},
  volume={32},
  number={1},
  pages={48--77},
  year={2002},
  publisher={SIAM}
}

@inproceedings{garivier2011upper,
  title={On upper-confidence bound policies for switching bandit problems},
  author={Garivier, Aur{\'e}lien and Moulines, Eric},
  booktitle={International conference on algorithmic learning theory},
  pages={174--188},
  year={2011},
  organization={Springer}
}

@inproceedings{lu2010contextual,
  title={Contextual multi-armed bandits},
  author={Lu, Tyler and P{\'a}l, D{\'a}vid and P{\'a}l, Martin},
  booktitle={Proceedings of the Thirteenth international conference on Artificial Intelligence and Statistics},
  pages={485--492},
  year={2010},
  organization={JMLR Workshop and Conference Proceedings}
}

@article{weiner2012search,
  title={In search of synergy: strategies for combining interventions at multiple levels},
  author={Weiner, Bryan J and Lewis, Megan A and Clauser, Steven B and Stitzenberg, Karyn B},
  journal={Journal of the National Cancer Institute Monographs},
  volume={2012},
  number={44},
  pages={34--41},
  year={2012},
  publisher={Oxford University Press}
}

@article{legare2018interventions,
  title={Interventions for increasing the use of shared decision making by healthcare professionals},
  author={L{\'e}gar{\'e}, France and Adekpedjou, Rh{\'e}da and Stacey, Dawn and Turcotte, St{\'e}phane and Kryworuchko, Jennifer and Graham, Ian D and Lyddiatt, Anne and Politi, Mary C and Thomson, Richard and Elwyn, Glyn and others},
  journal={Cochrane database of systematic reviews},
  number={7},
  volume={21},
  year={2018},
  publisher={John Wiley \& Sons, Ltd}
}

@inproceedings{li2010contextual,
  title={A contextual-bandit approach to personalized news article recommendation},
  author={Li, Lihong and Chu, Wei and Langford, John and Schapire, Robert E},
  booktitle={Proceedings of the 19th international conference on World wide web},
  pages={661--670},
  year={2010}
}

@article{auer2002finite,
  title={Finite-time analysis of the multiarmed bandit problem},
  author={Auer, Peter and Cesa-Bianchi, Nicolo and Fischer, Paul},
  journal={Machine learning},
  volume={47},
  number={2},
  pages={235--256},
  year={2002},
  publisher={Springer}
}

@inproceedings{schlisselberg2025delay,
  title={Delay as Payoff in MAB},
  author={Schlisselberg, Ofir and Cohen, Ido and Lancewicki, Tal and Mansour, Yishay},
  booktitle={Proceedings of the AAAI Conference on Artificial Intelligence},
  number={19},
  pages={20310--20317},
  year={2025}
}

@article{chen2018contextual,
  title={Contextual combinatorial multi-armed bandits with volatile arms and submodular reward},
  author={Chen, Lixing and Xu, Jie and Lu, Zhuo},
  journal={Advances in Neural Information Processing Systems},
  volume={31},
  year={2018}
}

@inproceedings{xu2020contextual,
  title={Contextual-bandit based personalized recommendation with time-varying user interests},
  author={Xu, Xiao and Dong, Fang and Li, Yanghua and He, Shaojian and Li, Xin},
  booktitle={Proceedings of the AAAI Conference on Artificial Intelligence},
  volume={34},
  pages={6518--6525},
  year={2020}
}

@inproceedings{basu2021contextual,
  title={Contextual blocking bandits},
  author={Basu, Soumya and Papadigenopoulos, Orestis and Caramanis, Constantine and Shakkottai, Sanjay},
  booktitle={International Conference on Artificial Intelligence and Statistics},
  pages={271--279},
  year={2021},
  organization={PMLR}
}

@inproceedings{fleisher2021s,
  title={What's fair about individual fairness?},
  author={Fleisher, Will},
  booktitle={Proceedings of the 2021 AAAI/ACM Conference on AI, Ethics, and Society},
  pages={480--490},
  year={2021}
}

@article{castelnovo2022clarification,
  title={A clarification of the nuances in the fairness metrics landscape},
  author={Castelnovo, Alessandro and Crupi, Riccardo and Greco, Greta and Regoli, Daniele and Penco, Ilaria Giuseppina and Cosentini, Andrea Claudio},
  journal={Scientific reports},
  volume={12},
  number={1},
  pages={4209},
  year={2022},
  publisher={Nature Publishing Group UK London}
}

@inproceedings{kassraie2022meta,
  title={Meta-learning hypothesis spaces for sequential decision-making},
  author={Kassraie, Parnian and Rothfuss, Jonas and Krause, Andreas},
  booktitle={International Conference on Machine Learning},
  pages={10802--10824},
  year={2022},
  organization={PMLR}
}

\end{document}